\documentclass[11pt, a4paper, onecolumn, copyright, goog]{google}

\usepackage{multirow}
\usepackage[authoryear, sort&compress, round]{natbib}

\crefname{appsec}{Appendix Section}{Appendix Sections}
\usepackage{amsmath}
\usepackage{amssymb}
\usepackage{mathtools}
\usepackage{amsthm}
\usepackage{algorithm}
\usepackage{algpseudocode}
\usepackage{wrapfig}

\newcommand{\methodname}{ECO}
\newcommand{\vect}[1]{\mathbf{#1}}
\newcommand{\vtheta}{\boldsymbol{\theta}}
\newcommand{\veta}{\boldsymbol{\eta}}
\newcommand{\vg}{\vect{g}}
\newcommand{\vm}{\vect{m}}
\newcommand{\vvv}{\vect{v}}
\newcommand{\ve}{\vect{e}}
\newcommand{\vx}{\vect{x}}
\newcommand{\vy}{\vect{y}}

\newcommand{\vu}{\vect{u}}
\newcommand{\vs}{\vect{s}}
\newcommand{\qfn}{\textit{q}}
\newcommand{\qfneco}{\textit{ECO\_QUANTIZE}}
\newcommand{\optfn}{\textit{OPTIM\_STEP}}
\newcommand{\meq}{\mathbin{=}}

\newcommand{\norm}[1]{\left\lVert#1\right\rVert_2}

\newcommand{\R}{\mathbb{R}}
\newcommand{\E}{\mathbb{E}}

\newcommand{\inner}[2]{\langle #1, #2 \rangle}

\theoremstyle{plain}
\newtheorem{theorem}{Theorem}[section]

\newtheorem{lemma}[theorem]{Lemma}

\theoremstyle{definition}
\newtheorem{definition}[theorem]{Definition}
\newtheorem{assumption}[theorem]{Assumption}
\theoremstyle{remark}

\uselogo{}

\title{ECO: Quantized Training without Full-Precision Master Weights}

\renewcommand{\today}{}

\author[1,2]{Mahdi Nikdan}
\author[1]{Amir Zandieh}
\author[2]{Dan Alistarh}
\author[1]{Vahab Mirrokni}

\correspondingauthor{Mahdi Nikdan (nikdanmahdi@gmail.com)}

\affil[1]{Google Research}
\affil[2]{ISTA}

\begin{abstract}
Quantization has significantly improved the compute and memory efficiency of Large Language Model (LLM) training. However, existing approaches still rely on accumulating their updates in high-precision: concretely, gradient updates must be applied to a high-precision weight buffer, known as \textit{master weights}. This buffer introduces substantial memory overhead, particularly for Sparse Mixture of Experts (SMoE) models, where model parameters and optimizer states dominate memory usage. To address this, we introduce the Error-Compensating Optimizer (ECO), which eliminates master weights by applying updates directly to quantized parameters. ECO quantizes weights after each step and carefully injects the resulting quantization error into the optimizer momentum, forming an error-feedback loop with no additional memory. We prove that, under standard assumptions and a decaying learning rate, ECO converges to a constant-radius neighborhood of the optimum, while naive master-weight removal can incur an error that is inversely proportional to the learning rate. We show empirical results for pretraining small Transformers (30--800M), a Gemma-3 1B model, and a 2.1B parameter Sparse MoE model with FP8 quantization, and fine-tuning DeepSeek-MoE-16B in INT4 precision. Throughout, ECO matches baselines with master weights up to near-lossless accuracy, significantly shifting the static memory vs validation loss Pareto frontier.
\end{abstract}

\begin{document}

\maketitle

\section{Introduction}
\begin{wrapfigure}{t}{0.5\textwidth}
    \centering
    \includegraphics[width=0.5\textwidth]{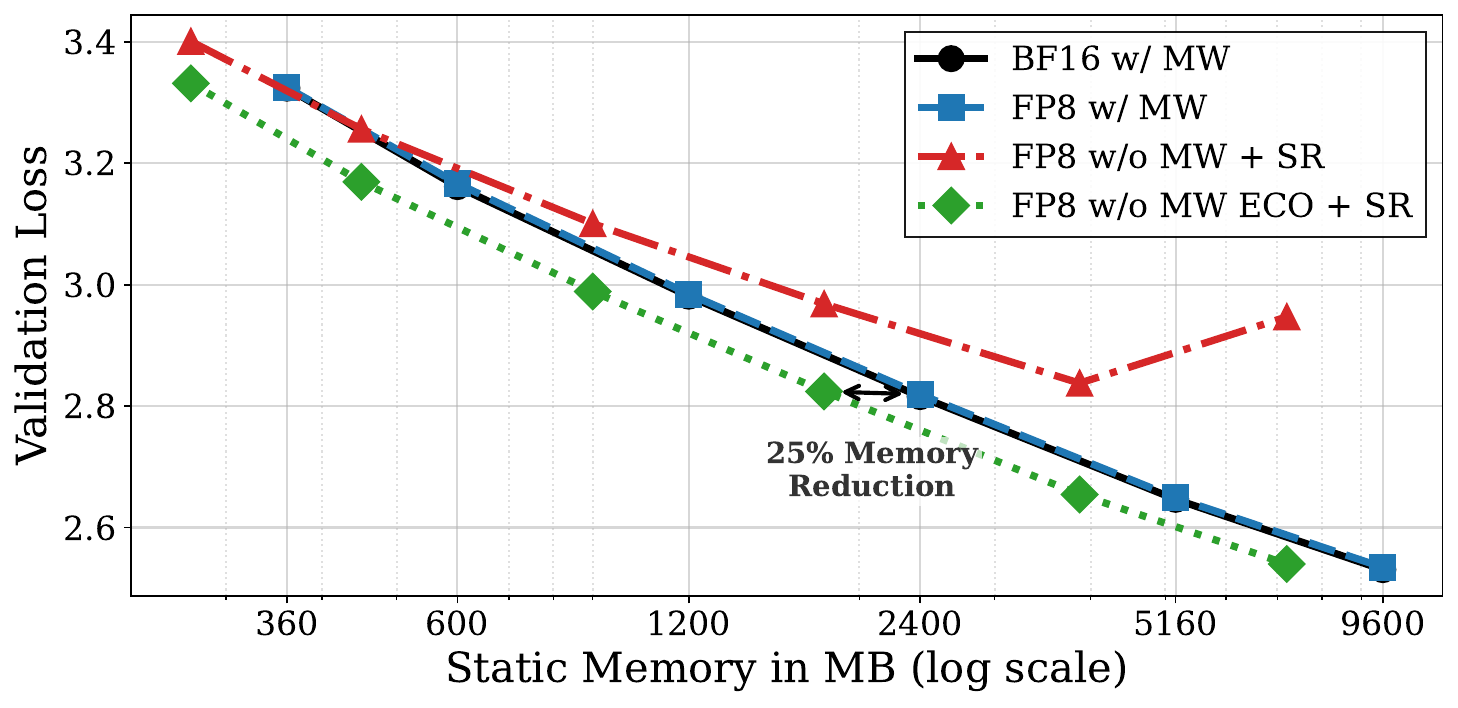}
    \caption{Static Memory Used vs Validation Loss comparing the standard BF16, FP8 with Master weights (FP8 w/ MW) baselines with standard stochastic rounding (FP8 w/o MW + SR) and ECO. \methodname{} with stochastic rounding (SR) provides a significantly better Pareto frontier. Gradient accumulation is disabled in all cases.}
    \label{fig:pareto}
\end{wrapfigure}
\label{sec:intro}

Scaling Large Language Model (LLM) training comes with substantial computational and memory costs. As models have grown from billions to trillions of parameters, training memory has become a central bottleneck. Low-precision training has therefore emerged as a practical direction: recent FP8 \citep{fp8lm, deepseekv3}, and even lower precision \citep{quest} training methods can reduce activation memory and accelerate training while maintaining stable optimization.

Despite this progress, a key overhead in quantized training remains untouched: the presence of \textit{master weights}. Most quantized and quantization-aware training pipelines still preserve a high-precision copy of the parameters (typically FP32) to accumulate gradient updates. This is largely because many updates are smaller than the discretization gap of low-precision formats: applying them directly to quantized weights can make updates vanish or incur large quantization noise. As a result, the model weight memory footprint often stays similar to the high-precision baseline, even when the forward and backward passes are heavily quantized. Even carefully engineered FP8 training systems explicitly retain high-precision accumulators for stability \citep{fp8lm, deepseekv3}. The issue is especially pronounced for Sparse Mixture of Experts (SMoE) models, where only a subset of parameters is active per token, yet \emph{all} master weights must reside in memory.

More broadly, attempts to avoid high-precision accumulation either do not scale to LLM training \citep{ondevice} or have only been effective in narrow settings \citep{elmo}. This leaves a clear gap: a general method that removes master weights without sacrificing convergence or introducing additional memory overhead. Eliminating master weights can yield memory savings comparable to quantizing optimizer states (e.g., momentum buffers), an approach that has been widely explored and is very popular~\citep{opt8bit}.

In this work, we introduce the \textbf{E}rror-\textbf{C}ompensating \textbf{O}ptimizer (\methodname{}), which enables accurate quantized training without full-precision master weights, and thus \textit{zero} extra memory overhead. The key idea is the following: after updating each layer's parameters, we quantize the updated weights and inject the resulting quantization error into {the optimizer's momentum buffer}. This creates an error-feedback loop that \textit{carries forward} the lost updates and compensates for them in subsequent steps, allowing updates to be applied directly to quantized parameters.

The resulting \methodname{} iteration is simple to implement and requires no extra hyperparameter tuning. It further comes with theoretical guarantees. We study the convergence behavior of \methodname{} applied to the SGD with momentum optimizer with momentum factor $\beta$. Under standard non-convex assumptions and a decaying learning rate, we prove that \methodname{} converges to a constant-radius neighborhood of the true optimum. Moreover, this radius is only a $\frac{1}{1-\beta^2}$ factor worse than the best achievable bound when using master weights, where a nonzero error is unavoidable because the solution must lie on the quantization grid. We further construct a quadratic example showing that this bound is tight up a constant factor. In the same example, we show that naively removing master weights (without momentum error injection) yields a stationary error that scales inversely with the learning rate, and therefore diverges as the learning rate decays to zero.

We evaluate \methodname{} with FP8 quantization across scaling law studies on small transformers (30M--800M parameters) \citep{quest,quartet}, pre-training a Gemma-3 1B \citep{gemma3} and an SMoE 2.1B model, and fine-tuning a DeepSeek-MoE-16B model \citep{deepseekmoe}. Across settings, \methodname{} nearly matches the validation loss of baselines that rely on master weights while significantly outperforming naive master weight removal. Furthermore \methodname{} can reduce static memory usage by up to 25\%, shifting the Pareto frontier between memory consumption and validation loss, as illustrated in~\Cref{fig:pareto}.

\section{Related Work}
\label{sec:related}

\paragraph{Quantized/Quantization-Aware Training.}
Quantization-aware training (QAT) aims to enable low-precision inference by simulating quantization effects on weights and optionally activations during training \citep{lsq,pact,quest,dorefa,qil,bitnet,effqat,llmqat}. Quantized training methods go further by quantizing the backward pass computation to accelerate training \citep{halo,quartet,alberttseng,fp4alltheway,deepseekv3}. Post-training quantization (PTQ) methods such as \citet{gptq,awq,quarot,spinquant,flatquant} are computationally cheaper, but they typically incur larger accuracy degradation than QAT, especially at very low precision. Despite these advances, most QAT frameworks still rely on high-precision master weights to accumulate updates. Even recent QAT training systems such as FP8-LM \citep{fp8lm}, DeepSeek-V3 \citep{deepseekv3}, and Kimi-K2 \citep{kimik2}, who have rigorously tuned their quantization scheme, explicitly keep high-precision accumulators to maintain stability. In this context, \methodname{} is complementary to existing QAT and quantized training methods: it targets the remaining dependence on master weights.

\paragraph{Efforts Towards Low-Precision Accumulation.}
Avoiding master weights has proven difficult outside restricted settings. FP8-LM reports that FP8 accumulation fails at large LLM scales \citep{fp8lm}. \citet{ondevice} show that with careful gradient rescaling, INT8 accumulators can be stable for small convolutional networks that fit within 256KB of memory. APT \citep{apt} varies accumulator bit-width across layers for edge-device training. Collage \citep{collage} replaces FP32 with two BF16 accumulators due to a hardware constraint. \citet{bf16sr} argue that stochastic rounding is important for BF16 accumulation, and ELMO \citep{elmo} applies stochastic rounding to reduce the accumulator precision of the LLM head layer to BF16/FP8. Overall, there exists no general approach that enables sub-16-bit accumulation for large-scale LLM training, leaving an important gap that \methodname{} addresses.

\paragraph{Optimizer State Quantization.}
A related line of work quantizes optimizer states (e.g., first and second moments) rather than model weights. In practice, the first moment is often more tolerant to quantization than the second. FP8-LM \citep{fp8lm} reports that the first moment can be quantized to FP8 without difficulty. Other approaches quantize both moments to 8-bit \citep{opt8bit,scalingfp8,coat}, and \citet{opt4bit} pushes this to 4-bit for both buffers. \methodname{} targets a different bottleneck: the master-weight copy. This provides memory savings comparable to optimizer-state quantization, while remaining largely unexplored.

\paragraph{Error Feedback.}
Error feedback (EF) methods were developed to mitigate bias from compressed or quantized gradients, particularly in distributed optimization. They accumulate quantization residuals locally and add them back in later steps, preserving the sum of updates over time \citep{onebitsgd,onebitadam,zeropp,ef21}. \citet{ef21} provides a principled EF formulation and shows that it can match full-precision SGD convergence under appropriate assumptions. Directly applying EF to the master weight quantization requires storing an error buffer, which conflicts with memory reduction goals when training at scale. \methodname{} instead reuses the optimizer momentum buffer to store quantization error, achieving error feedback without any extra memory.

\section{Method}
In this section, we start by introducing the notation and covering relevant background. We then describe our main method \methodname{}. Finally, we present our theoretical results which analyze the convergence of \methodname{}.

\subsection{Notation and Background}
\paragraph{Notation.}
\label{sec:notation}
Throughout this section, we denote the model parameters by $\vtheta$ and their corresponding gradients by $\vg$.
The optimizer's first and second momentum buffers are represented by $\vm$ and $\vvv$, respectively, with their corresponding coefficients denoted by $\beta_1$ and $\beta_2$ (or just $\beta$ in case of SGD).
We denote the quantization as $\qfn(\cdot)$, and $\ve$ represents the quantization error (e.g., $\ve_{\vtheta} = \vtheta - \qfn(\vtheta)$), and $\eta$ is the learning rate.

\paragraph{Quantization.}
Quantization is the process of mapping continuous or high-precision values to a low-precision representation, primarily to reduce memory usage and enhance arithmetic throughput. This process typically involves an affine transformation (scaling by $s$ and shifting by $z$) to project the original values into the target range, followed by a rounding function that maps each value to the nearest grid point.

More formally, a high-precision vector $\vx$ is quantized to a low-precision vector $\vy$ using the formula $\vy \meq \textit{round}(\frac{\vx - z}{s})$. The original values can then be approximated using $\hat{\vx} \meq s \vy + z$. Thus, the fully reconstructed vector $\hat{\vx}_{z,s}$ is calculated as:
\begin{equation}
    \hat{\vx}_{z,s} = s \cdot \textit{round}(\frac{\vx-z}{s}) + z \cdot 
\end{equation}
Assuming the largest quantized value representable by the quantization format is $\rho$, then a standard choice for the scaling factor is $s \meq \max |\vx|/\rho$, which prevents overflow. It is also common to fix the zero-point $z \meq 0$, particularly for tensors in LLM training that are often near zero-mean. Therefore, for simplicity, when $z$ and $s$ are not explicitly mentioned, we assume this symmetric scheme, i.e., $\hat{\vx} \meq q(\vx) \meq \frac{\max{|\vx|}}{\rho} \cdot \textit{round}(\frac{\rho \vx}{\max{|\vx|}})$.

Quantization schemes can be categorized in several ways. One key distinction is their granularity, which defines which parts of an input tensor share the same quantization parameters (i.e., zero-point $z$ and scale $s$). For example, in row-wise quantization, an independent $z$ and $s$ are computed and applied to each row of an input matrix. Other methods exists, such as 1D or 2D group-wise quantization, where blocks or groups of elements within the tensor share quantization parameters \citep{deepseekv3,jetfire,mx}. 

Another categorization stems from the rounding function. A standard choice is round-to-nearest, which deterministically maps each value to its closest grid point. Alternatively, stochastic rounding maps a value to one of the two nearest grid points, where the probability of selecting either point is proportional to the distance to the other point. Round-to-nearest minimizes the magnitude of the error, while stochastic rounding results in an unbiased estimator.

\paragraph{Quantization-Aware Training with Master Weights.}
Most quantized LLM training pipelines keep high-precision master weights (typically FP32) as the update accumulator. At each step, the master weights are quantized to obtain low-precision weights used for the forward/backward pass, while gradients and optimizer updates are accumulated in the high-precision copy. This stabilizes training by preserving small updates, but it substantially limits the weight-memory savings of quantization: the full master-weight buffer must remain on memory throughout training.

\subsection{\methodname{}}
The high-level idea of~\methodname{} is to \textit{inject} the quantization error from the current step into the optimizer's momentum buffer. This mechanism ensures that the error from the current step is \textit{carried over} and incorporated into the parameter update of the subsequent step, effectively creating an error feedback loop. \Cref{alg:sgd+eco} provides a general overview, while \Cref{alg:qeco_sgd} and \Cref{alg:qeco_adam} detail the error injection process for the SGD with Momentum (SGDM) and Adam optimizers, respectively.

\paragraph{SGDM.}
ECO applies SGDM updates directly to the quantized weights. Concretely, at step $t$ with low-precision parameters $\hat{\vtheta}_t$, it forms a temporary iterate $\tilde{\vtheta}_{t+1}=\hat{\vtheta}_t+\vu_t$ (where $\vu_t$ is the SGDM update, dominated by momentum), quantizes it to obtain $\hat{\vtheta}_{t+1}=\qfn(\tilde{\vtheta}_{t+1})$, and defines the quantization error $\ve_{t+1} := \tilde{\vtheta}_{t+1}-\hat{\vtheta}_{t+1}$. ECO then injects this error into the momentum buffer so that the update lost due to quantization is carried forward and recovered in later steps.

We prove in Appendix \ref{apx:exact-error-injection} that, if the errors are injected into momentum as
$$
\vm \leftarrow \vm + \frac{1}{\eta}\,\ve_{t} - \frac{1}{\eta\beta}\,\ve_{t+1},
$$
then the resulting optimization trajectory is \emph{identical} to SGDM with master weights. The difficulty is that this exact rule is not memory-efficient: while $\ve_{t+1}$ is available on-the-fly from the current quantization, the previous-step residual $\ve_{t}$ must be stored, which reintroduces a persistent buffer.

We tackle this issue by a heuristic observation: $\ve_{t+1}$ and $\ve_{t}$ are typically close. Intuitively, assuming a fixed scale parameter, $\hat{\vtheta}_t$ is already on-grid, so moving to the next iterate only quantizes the \emph{increment} $\vu_t$, i.e., $\qfn(\hat{\vtheta}_t + \vu_t) = \hat{\vtheta}_t + \qfn(\vu_t)$. Since $\vu_t$ is dominated by momentum, it changes slowly from one step to the next, which in turn makes the induced quantization errors $\ve_{t+1}$ and $\ve_t$ close.
We also validate this empirically in \Cref{sec:exps-scaling}.
We therefore substitute $\ve_{t}\approx \ve_{t+1}$, yielding the memory-free injection rule
$$
\vm \leftarrow \vm + \frac{1}{\eta}\Bigl(1-\frac{1}{\beta}\Bigr)\ve_{t+1},
$$
which removes the need for either master weights or a stored error buffer. See \Cref{alg:qeco_sgd} for more details. Notably, we use this heuristic only to motivate the injection rule; later in this section, we provide a rigorous theoretical analysis of the resulting memory-efficient form.

\paragraph{Adam.}
We treat Adam in the same way as SGDM, except that Adam applies an \emph{adaptive}, element-wise learning rate. Adam's parameter update can be written in the form
$$
\vtheta_{t+1} = \vtheta_t - \eta \frac{\frac{\vm_{t+1}}{1-\beta_1^t}}{\sqrt{\frac{\vvv_{t+1}}{1-\beta_2^t}} + \epsilon},
$$
where $\vm_{t+1}$ and $\vvv_{t+1}$ are the first and second momentum buffers after incorporating the gradient at step $t$, and $\epsilon$ prevents division by zero. We identify the element-wise adaptive step size as
$$
\veta_t
\coloneq
\frac{\eta}{(1-\beta_1^t) (\sqrt{\frac{\vvv_{t+1}}{1-\beta_2^t}} + \epsilon)}.
$$
With this formulation, \methodname{}'s injection differs from the SGDM case only by replacing the scalar learning rate with Adam's element-wise effective step size. See \Cref{alg:qeco_adam}.

\begin{algorithm}[ht]
\caption{Quantized Training Step $t$ with \methodname{}}
\label{alg:sgd+eco}
\begin{algorithmic}[1]
\Require Quantized parameters $\hat{\vtheta}_t$
\Require Optimizer state $\hat{\vs}_t$, hyperparameters $H$
\Require Optimizer step function: $\optfn$
\Require \methodname{} quantization function: $\qfneco$
\State $\tilde{\vtheta}_{t+1}, \tilde{\vs}_{t+1} \leftarrow \optfn(\hat{\vtheta}_{t}, \hat{\vs}_{t}, H)$ \Comment{optimization step on quantized parameters}
\State $\hat{\vtheta}_{t+1}, \hat{\vs}_{t+1} \leftarrow \qfneco(\tilde{\vtheta}_{t+1}, \tilde{\vs}_{t+1}, H)$ \Comment{quantization + momentum injection}
\State \Return $\hat{\vtheta}_{t+1}, \hat{\vs}_{t+1}$
\end{algorithmic}
\end{algorithm}

\begin{algorithm}[ht]
\caption{$\qfneco$ for SGD with Momentum}
\label{alg:qeco_sgd}
\begin{algorithmic}[1]
\Require High-precision parameters $\tilde{\vtheta}_{t+1}$
\Require Optimizer state $\tilde{\vs}_{t+1}$, hyperparameter $H$
\State $\hat{\vtheta}_{t+1} \leftarrow \qfn(\tilde{\vtheta}_{t+1})$ \Comment{quantize the weights}
\State $\ve_{t+1} \leftarrow \tilde{\vtheta}_{t+1}-\hat{\vtheta}_{t+1}$ \Comment{compute the quantization error}
\State $\{\tilde{\vm}_{t+1}\} \leftarrow \tilde{\vs}_{t+1}$ \Comment{read momentum buffer from the optimizer state}
\State $\{\eta, \beta\} \leftarrow H$ \Comment{read SGDM hyperparameters}
\State $\hat{\vm}_{t+1} \leftarrow \tilde{\vm}_{t+1} + \frac{1}{\eta}(1-\frac{1}{\beta}) \ve_{t+1}$ \Comment{inject the quantization error into momentum}
\State \Return $\hat{\vtheta}_{t+1}, \{\hat{\vm}_{t+1}\}$
\end{algorithmic}
\end{algorithm}

\begin{algorithm}[ht]
\caption{$\qfneco$ for Adam}
\label{alg:qeco_adam}
\begin{algorithmic}[1]
\Require High-precision parameters $\tilde{\vtheta}_{t+1}$
\Require Optimizer state $\tilde{\vs}_{t+1}$, hyperparameter $H$
\State $\hat{\vtheta}_{t+1} \leftarrow \qfn(\tilde{\vtheta}_{t+1})$ \Comment{quantize the weights}
\State $\ve_{t+1} \leftarrow \tilde{\vtheta}_{t+1}-\hat{\vtheta}_{t+1}$ \Comment{compute the quantization error}
\State $\{\tilde{\vm}_{t+1}, \vvv_{t+1}\} \leftarrow \tilde{\vs}_{t+1}$ \Comment{read momentum buffers from the optimizer state}
\State $\{\eta, \beta_1, \beta_2, \epsilon\} \leftarrow H$ \Comment{read Adam hyperparameters}
\State $\hat{\vm}_{t+1} \leftarrow \tilde{\vm}_{t+1} + \frac{1 - \beta_1^t}{\eta}(1-\frac{1}{\beta_1}) (\sqrt{\frac{\vvv_{t+1}}{1 - \beta_2^t}} + \epsilon) \odot \ve_{t+1}$ \Comment{inject the quantization error into momentum}
\State \Return $\hat{\vtheta}_{t+1}, \{\hat{\vm}_{t+1}, \vvv_{t+1}\}$
\end{algorithmic}
\end{algorithm}

\subsection{Convergence Analysis}
This section presents the convergence analysis for the SGDM variant of the ECO optimizer. By constructing a virtual sequence, we prove that the algorithm converges to a near stationary point. All proofs are given in Appendix \ref{apx:convergence-proofs}.

\subsubsection{Setup and Algorithm}

We consider the optimization problem $\min_{\vtheta \in \R^d} f(\vtheta)$, where $f$ is $L$-smooth and bounded below by $f^*$.

The \methodname{} Optimizer updates are expanded as follows:
\begin{align}
    \tilde{\vm}_{t+1} &= \beta \hat{\vm}_t + (1-\beta) \nabla f(\hat{\vtheta}_t) \label{eq:m_tilde} \\
    \tilde{\vtheta}_{t+1} &= \hat{\vtheta}_t - \eta \tilde{\vm}_{t+1} \label{eq:theta_tilde} \\
    \hat{\vtheta}_{t+1} &= q(\tilde{\vtheta}_{t+1}) \label{eq:quant} \\
    \ve_{t+1} &= \tilde{\vtheta}_{t+1} - \hat{\vtheta}_{t+1} \label{eq:error} \\
    \hat{\vm}_{t+1} &= \tilde{\vm}_{t+1} + \alpha \ve_{t+1} \label{eq:m_hat}
\end{align}
where $\eta$ is the learning rate, $\beta \in [0,1)$ is the momentum parameter, and the error injection strength is set to:
\begin{equation}
    \alpha = \frac{1}{\eta}\left(1 - \frac{1}{\beta}\right). \label{eq:alpha}
\end{equation}

\subsubsection{Assumptions}

We rely on the following standard assumptions for non-convex optimization analysis.

\begin{assumption}[L-Smoothness]
The function $f$ is $L$-smooth, i.e., $\|\nabla f(x) - \nabla f(y)\| \le L \|x-y\|$ for all $x, y$.
\end{assumption}

\begin{assumption}[Unbiased Quantization with Bounded Error Variance]
The quantization error is zero-mean with bounded variance $\sigma^2$: $\E[\ve_t] = 0$ and $\E[\|\ve_t\|^2] \le \sigma^2$.
\end{assumption}

\begin{assumption}[Bounded Gradient]
There exists $G > 0$ such that $\|\nabla f(\vtheta)\| \le G$ for all $\vtheta$.
\end{assumption}

\subsubsection{Virtual Sequence Analysis}

Following the methodology of \citet{ef21}, we construct a ``virtual sequence'' $\vtheta_t$.

\begin{definition}[Virtual Sequence]
Define the virtual sequence $\vtheta_t$ as:
\begin{equation}
\vtheta_t \coloneqq \hat{\vtheta}_t - \frac{\eta \beta}{1-\beta} \hat{\vm}_t.
\end{equation}
\end{definition}

\begin{lemma}[Virtual Sequence Dynamics]
\label{lemma:virtual_dynamics}
The virtual sequence $\vtheta_t$ evolves as:
\begin{equation}
    \vtheta_{t+1} = \vtheta_t - \eta \nabla f(\hat{\vtheta}_t),
\end{equation}
\end{lemma}
This lemma demonstrates that by tracking this specific combination of weights and momentum, we can analyze the \methodname{} trajectory as a standard gradient descent process on the loss surface.

\subsubsection{Descent and Momentum Bounds}

We derive a descent inequality for the virtual sequence and bound the momentum term which accumulates the quantization error.

\begin{lemma}[Descent Lemma]
\label{lemma:descent}
Let $C = \frac{\eta \beta}{1-\beta}$. For $\eta \le \frac{1}{2L}$, the virtual sequence satisfies:
\begin{equation} \label{eq:descent_lemma_main}
    f(\vtheta_{t+1}) \le f(\vtheta_t) - \frac{\eta}{4} \norm{\nabla f(\hat{\vtheta}_t)}^2 + \frac{\eta L^2 C^2}{2} \norm{\hat{\vm}_t}^2
\end{equation}
\end{lemma}
This allows us to control the dynamics of the optimization trajectory.

\begin{lemma}[Bounded Momentum]
\label{lemma:momentum_bound}
Under the assumptions, the squared norm of the momentum $\hat{\vm}_t$ is bounded in expectation by a constant $M^2$. Specifically, for all $t$:
\begin{equation}
    \E[\|\hat{\vm}_t\|^2] \le M^2 \coloneqq 2G^2 + \frac{2\alpha^2 \sigma^2}{1-\beta^2}.
\end{equation}
\end{lemma}
This ensures that the quantization error injected into the momentum buffer does not explode, keeping the optimization stable.

\subsubsection{Convergence Theorem}

\begin{theorem}[Convergence Rate]
\label{thm:convergence}
For $\eta \le \frac{1}{2L}$, the ECO optimizer converges to a neighborhood:
\begin{equation}
    \min_{t \in \{0, \dots, T-1\}} \E \left[\|\nabla f(\hat{\vtheta}_t)\|^2 \right] \le \frac{4(f(\vtheta_0) - f^*)}{\eta T} + \sigma_{\text{quant}}^2,
\end{equation}
where the quantization noise floor $\sigma_{\text{quant}}^2$ is given by:
\begin{equation}
    \sigma_{\text{quant}}^2 = \frac{4 \eta^2 \beta^2 L^2 G^2}{(1-\beta)^2} + \frac{4L^2 \sigma^2}{1-\beta^2}
\end{equation}
\end{theorem}

\noindent \textbf{Discussion on Decaying Learning Rate:} As $\eta \to 0$, the noise floor $\sigma_{\text{quant}}^2$ becomes:
\begin{equation}
    \lim_{\eta \to 0} \sigma_{\text{quant}}^2 = \frac{4L^2 \sigma^2}{1-\beta^2} \label{eq:noise-floor}
\end{equation}
While the noise floor persists even as the learning rate vanishes, we show in the next subsection that this noise floor is tight up to the constant $4$. Additionally, we note that even with master weights, since the final solution must lie on the quantization grid, a noise floor of $L^2 \sigma^2$ is unavoidable.

\subsubsection{Deterministic Rounding}
We now provide a similar study where deterministic round-to-nearest is used instead of stochastic rounding. In this case, the zero-mean error assumption (Assumption 3.2) is violated. We instead assume a bounded deterministic error $\|\ve_t\| \le \delta$ for all $t$.

\begin{lemma}[Deterministic Momentum Bound]
\label{lemma:momentum_bound_det}
Under the deterministic error assumption $\|\ve_t\| \le \delta$ and bounded gradients $\|\nabla f(\vtheta)\| \le G$, the norm of the injected momentum buffer in \methodname{} is uniformly bounded for all $t$:
\begin{equation}
    \|\hat{\vm}_t\| \le M_{\text{det}} \coloneqq G + \frac{|\alpha| \delta}{1-\beta}.
\end{equation}
\end{lemma}

\begin{theorem}[Deterministic Convergence]
\label{thm:convergence_det}
For $\eta \le \frac{1}{2L}$, the \methodname{} optimizer with deterministic rounding converges to a neighborhood of the optimum:
\begin{equation}
    \min_{t < T} \|\nabla f(\hat{\vtheta}_t)\|^2 \le \frac{4(f(\vtheta_0) - f^*)}{\eta T} + \Gamma_{\text{quant}}^2
\end{equation}
where the deterministic noise floor is defined as:
\begin{equation}
    \Gamma_{\text{quant}}^2 = 2L^2 C^2 M_{\text{det}}^2 = \frac{2L^2 \eta^2 \beta^2}{(1-\beta)^2} \left( G + \frac{|\alpha|\delta}{1-\beta} \right)^2.
\end{equation}
\end{theorem}

\paragraph{Comparison of Noise Floors.}
It is instructive to compare the noise floor of the stochastic case ($\sigma_{\text{quant}}^2$) and the deterministic case ($\Gamma_{\text{quant}}^2$) as the learning rate $\eta \to 0$. In the stochastic case, the noise floor remains constant at $\mathcal{O}(L^2 \sigma^2 / (1-\beta^2))$. In the deterministic case, substituting $|\alpha| = (1-\beta)/\eta \beta$ results in a floor of $\mathcal{O}(L^2 \delta^2 / (1-\beta)^2)$. Assuming $\sigma \approx \delta$, the deterministic bound is significantly larger due to the $(1-\beta)^{-2}$ dependence, reflecting the fact that systematic biases in quantization are harder for the momentum buffer to ``average out'' than zero-mean noise.

\subsection{Lower-Bound on  Worst-Case Behavior}
\label{sec:fundamental}
We analyze the optimization dynamics on a one-dimensional quadratic objective $f(x) = \frac{L}{2} x^2$ with $L > 0$. The gradient is $\nabla f(x) = Lx$. We assume a stochastic quantization model where the quantized value $\hat{x} = q(x)$ satisfies $\hat{x} = x + \xi$, with $\xi$ being zero-mean noise independent of $x$ and $\E[\xi^2] = \sigma^2$. We examine the expected squared gradient norm of the stationary \textit{quantized} parameters, defined as $\mathcal{L} = \lim_{t \to \infty} \E[(\nabla f(\hat{x}_t))^2]$, in the limit as the learning rate $\eta \to 0$. The results are summarized below, while the formal derivations are deferred to Appendix \ref{apx:fundamental}.

\paragraph{SGDM with Master Weights.}
In this standard setting, the master weights evolve in high precision, but the gradient is computed using the quantized weights. Master weights allow the underlying parameter to converge to the true optimum. However, the quantized weights are $\xi$ away from the master weights. Consequently, the error is dominated by the quantization resolution:
\begin{equation}
    \lim_{\eta \to 0} \mathcal{L}_{\text{MW}} = L^2 \sigma^2.
\end{equation}

\paragraph{Naive Master Weight Removal.}
When master weights are removed, the update is applied directly to the quantized parameter: $\hat{x}_{t+1} = q(\hat{x}_t - \eta m_{t+1})$. This process reaches a stationary distribution, however, the variance is inversely proportional to the learning rate:
\begin{equation}
    \mathcal{L}_{\text{Naive}} \propto \frac{1}{\eta} \xrightarrow{\eta \to 0} \infty.
\end{equation}
This confirms that without error compensation, one cannot achieve high accuracy by annealing the learning rate.

\paragraph{ECO.}
ECO stabilizes the master-weight-free training by injecting quantization noise into the momentum buffer. In the limit of small learning rates, the process converges to a stationary distribution determined by the noise accumulation in the momentum term:
\begin{equation}
    \lim_{\eta \to 0} \mathcal{L}_{\text{ECO}} = \frac{L^2 \sigma^2}{1-\beta^2}.
\end{equation}
This shows that \methodname{} prevents the $1/\eta$ explosion seen in the naive case. Additionally, this verifies that the noise floor in in \Cref{eq:noise-floor} is tight up to a factor of $4$.

\section{Experiments}
\label{sec:exps}

\begin{table}[t]
    \centering
    \caption{Validation loss comparison across model sizes 30-800M, with ``dvg'' denoting divergence. $^*$``N/A'': one entry is unavailable due to data loss.}
    \label{tab:loss_values}
    \begin{tabular}{lcccccc}
        \toprule
        \textbf{Model Size} & \textbf{30M} & \textbf{50M} & \textbf{100M} & \textbf{200M} & \textbf{430M} & \textbf{800M}\\
        \midrule
        BF16 w/ MW                        & 3.3238 & 3.1616 & 2.9811 & 2.8157 & 2.6464 & 2.5306 \\
        FP8 w/ MW + RTN                  & 3.3248 & 3.1668 & 2.9846 & 2.8194 & 2.6490 & 2.5343 \\
        FP8 w/ MW + SR              & 3.3309 & 3.1719 & 2.9884 & 2.8231 & 2.6500 & N/A$^*$ \\
        \midrule
        FP8 w/o MW + RTN             & dvg & dvg & dvg & dvg & dvg & dvg \\
        FP8 w/o MW + SR             & 3.4008 & 3.2563 & 3.1006 & 2.9684 & 2.8378 & 2.9471 \\
        FP8 w/o MW ECO + RTN         & 3.3640 & 3.1862 & 3.0025 & 2.8776 & 2.7237 & 2.6046 \\
        FP8 w/o MW ECO + SR         & \textbf{3.3317} & \textbf{3.1695} & \textbf{2.9888} & \textbf{2.8241} & \textbf{2.6544} & \textbf{2.5399} \\
        \bottomrule
    \end{tabular}
\end{table}

\subsection{Baselines}
We evaluate the following baselines that use high-precision accumulation.

\begin{itemize}
    \item \textbf{FP32 accumulation with BF16 computation (BF16 w/ MW)}:
    This configuration serves as the reference baseline.
    Training is performed using FP32 master weights, while operands are cast to BF16 prior to each matrix multiplication to improve efficiency.
    This setup follows standard automatic mixed-precision training \citep{amp} and provides an upper bound on achievable performance.

    \item \textbf{FP32 accumulation with FP8 round-to-nearest forward pass (FP8 w/ MW + RTN)}:
    This quantization-aware training (QAT) baseline quantizes both weights and activations to the FP8 E4M3 format during the forward pass using round-to-nearest.
    Row-wise scaling is applied, with each scale set to the maximum absolute value in the corresponding row.
    Prior work has shown that this approach is largely lossless \citep{deepseekv3, fp8lm}.

    \item \textbf{FP32 accumulation with FP8 stochastic rounding forward pass (FP8 w/ MW + SR)}:
    This baseline is identical to the previous one, except that weights are quantized using stochastic rounding.
    Activations remain quantized with round-to-nearest.
\end{itemize}

The baselines above maintain FP32 master weights and therefore establish upper bounds for the following methods, which eliminate master weight storage.

\begin{itemize}
    \item \textbf{FP8 accumulation and forward pass with round-to-nearest (FP8 w/o MW + RTN)}:
    This baseline provides a direct comparison to \methodname{}.
    No high-precision master weights are stored.
    After each parameter update, weights are quantized to FP8 using round-to-nearest.
    Activations are also quantized to FP8.

    \item \textbf{FP8 accumulation and forward pass with stochastic rounding (FP8 w/o MW + SR)}:
    This method mirrors the previous baseline, but applies stochastic rounding to the weights.
    Activations are still quantized using round-to-nearest. This corresponds to the approach suggested by \citet{bf16sr}.

    \item \textbf{FP8 accumulation and forward pass with round-to-nearest and \methodname{} (FP8 w/o MW \methodname{} + RTN)}:
    In addition to removing master weights and applying round-to-nearest quantization to both weights and activations, this method incorporates our momentum injection mechanism to mitigate quantization error.

    \item \textbf{FP8 accumulation and forward pass with stochastic rounding and \methodname{} (FP8 w/o MW \methodname{} + SR)}:
    This variant is identical to the previous method, but uses stochastic rounding for weight quantization.
\end{itemize}

\subsection{Scaling Law Experiments}
\label{sec:exps-scaling}
\paragraph{Setting.}
We evaluate \methodname{} using a pre-training scaling study, following \citet{quest}.
We train models with sizes of 30M, 50M, 100M, 200M, 430M, and 800M parameters.
For a model with $N$ parameters, training is performed on $100N$ tokens from the C4 dataset \citep{t5}, corresponding to $5\times$ the Chinchilla-optimal token count \citep{chinchilla}.
We use the T5 tokenizer \citep{t5,sentencepiece}.
Both the batch size and sequence length are fixed to 512.
We use the AdamW optimizer with $(\beta_1, \beta_2, \epsilon) = (0.9, 0.98, 10^{-9})$.
The learning rate is linearly warmed up from $0.01\times$ the peak value to the peak over the first $10\%$ of training, followed by cosine decay to $0.1\times$ the peak.
We apply a weight decay of 0.1 and gradient clipping with a norm of 1.0.
Refer to \citet{quest} for more details on the hyperparameters.
For quantized runs, we apply the method only to the linear layers within transformer blocks, excluding the embedding and output layers.

\paragraph{Results.}
\Cref{tab:loss_values} reports the final validation loss achieved by each method.
The results show that \methodname{} substantially improves over naive removal of master weights.
When stochastic rounding is used, \methodname{} nearly recovers the performance of methods that retain master weights.
As expected, the gains are smaller with round-to-nearest quantization, since it introduces bias into the momentum buffer.

\paragraph{Memory and Runtime.} In addition, \Cref{fig:pareto} shows that \methodname{} establishes a new static memory--loss Pareto frontier, offering significantly lower memory usage for a given validation loss. Regarding runtime, the injection is a simple element-wise operation and adds negligible overhead.

\paragraph{Study on the Similarity of 
Consecutive Errors.}

\begin{wrapfigure}{t}{0.5\textwidth}
    \centering
    \includegraphics[width=0.48\textwidth]{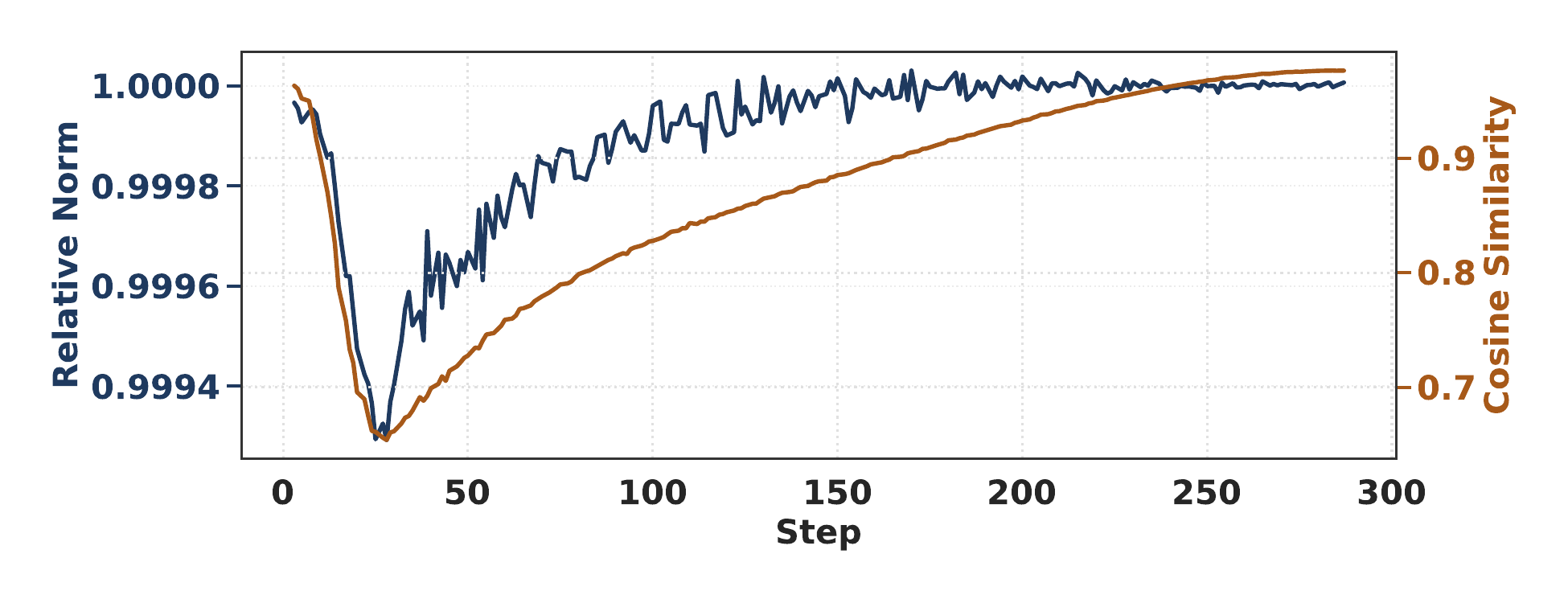}
    \caption{Similarity of consecutive quantization errors. Left: relative norm $\norm{\ve_{t+1}}/\norm{\ve_t}$. Right: cosine similarity between $\ve_t$ and $\ve_{t+1}$.}
    \label{fig:consecutive-errs}
\end{wrapfigure}

We repeat the 30M experiment with master weights and round-to-nearest (RTN), and measure the similarity between consecutive quantization errors. Specifically, we track the relative norm $\frac{\norm{\ve_{t+1}}}{\norm{\ve_t}}$ and the cosine similarity between $\ve_t$ and $\ve_{t+1}$ throughout training. \Cref{fig:consecutive-errs} reports both metrics. The relative norm remains close to $1$ during training, indicating that $\norm{\ve_t}$ varies slowly over time, and the cosine similarity stays consistently high, indicating strong alignment between consecutive errors. The observed trend follows the learning-rate schedule: larger learning rates lead to larger differences between consecutive errors, while these differences diminish as the learning rate decays.

\subsection{Gemma 3 1B Pre-training}

\paragraph{Setting.}
We pre-train the Gemma 3 1B model \citep{gemma3} from scratch on 40B tokens from the C4 dataset \citep{t5}.
The batch size is 256 and the sequence length is 512.
We use the publicly available Gemma 3 tokenizer.
Training uses the AdamW optimizer with the same hyperparameters as in the scaling law experiments.
The learning rate peaks at $10^{-4}$, with a linear warmup from $10^{-6}$ over the first $10\%$ of training, followed by cosine decay to $10^{-5}$.

\paragraph{Results.}
\Cref{fig:gemma_smoe_deepseek} (Left) compares the final validation loss across methods.
The results confirm the effectiveness of \methodname{}, particularly when combined with stochastic rounding.

\subsection{Mixture of Experts Pre-training}

\paragraph{Setting.}
We pre-train a sparse mixture-of-experts (SMoE) model with 2.1B total parameters.
The model contains 32 experts, of which 4 are activated per token.
It consists of 24 transformer layers, each with a hidden dimension of 576, an intermediate dimension of 2304, and 9 attention heads.
Training uses $100\times$ the number of active parameters in tokens from the LM1B dataset \citep{lm1b}.
We reuse the T5 tokenizer \citep{t5, sentencepiece}.
Optimization is performed with AdamW, using a weight decay of $0.1$, and a learning rate that increases linearly from $2\times10^{-6}$ to $2\times10^{-5}$ over the first $1\%$ of training, followed by cosine decay back to $2\times10^{-6}$.
The batch size is 256 and the sequence length is 512. For the quantized runs, we only quantize the expert linear layers.

\paragraph{Results.}
\Cref{fig:gemma_smoe_deepseek} (Left) summarizes the final validation loss for each method.
Consistent with prior experiments, \methodname{} clearly outperforms naive master weight removal, while incurring only a minimal loss compared to approaches that retain master weights.

\paragraph{Discussion on Memory.}
Due to the SMoE model architecture, the memory required for activation storage is substantially smaller than that required for weights.
With activation checkpointing enabled and no gradient accumulation, peak memory usage is dominated by master weights and optimizer states.
Reducing master weight precision from FP32 to FP8 therefore lowers peak memory consumption from $12$ bytes per parameter to $9$, a reduction of approximately $25\%$.

\begin{figure}[t]
    \centering
    \begin{subfigure}{0.48\textwidth}
        \raisebox{14pt}{\includegraphics[width=\linewidth]{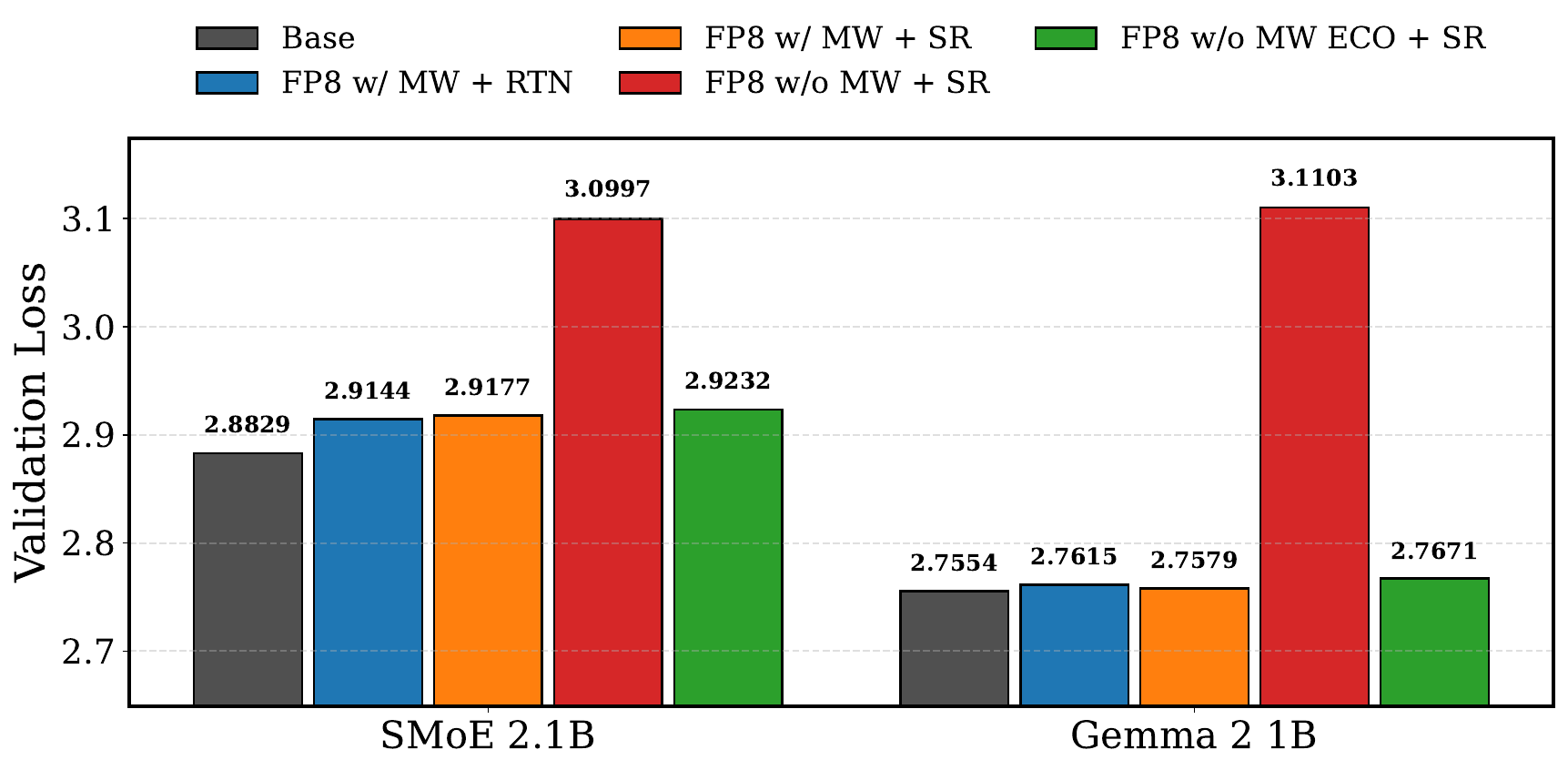}}
    \end{subfigure}
    \hfill
    \begin{subfigure}{0.48\textwidth}
        \includegraphics[width=\linewidth]{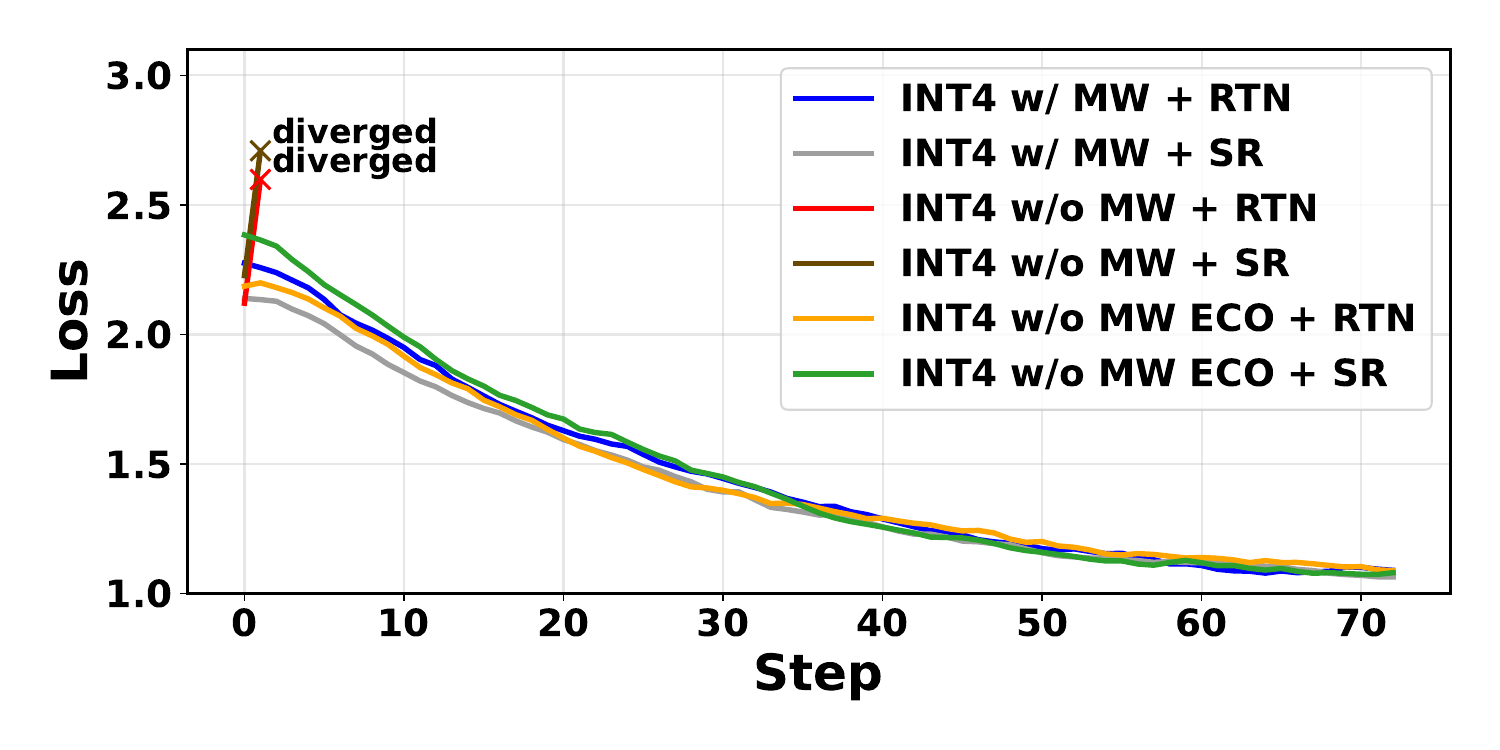}
    \end{subfigure}
    \vspace*{-14pt}
    \caption{(Left) Gemma 3 1B and SMoE 2.1B validation loss comparison, and (Right) Smoothed training loss during fine-tuning of DeepSeek-MoE-16B-Base \citep{deepseekmoe}.}
    \label{fig:gemma_smoe_deepseek}
\end{figure}

\subsection{DeepSeek-MoE-16B Fine-tuning}
\paragraph{Setting.}
We apply \methodname{} to tensor-wise INT4 weight-only QAT of DeepSeek-MoE-16B-Base \citep{deepseekmoe}. The model has 64 experts, with 8 active experts per token (approximately 2.8B parameters), including 2 shared experts. We fine-tune on the OpenAssistant-Guanaco dataset \citep{qlora} for 3 epochs with sequence length 2048, using AdamW with micro-batch size 1 and gradient accumulation of 16. The learning rate is linearly warmed up from $2\times 10^{-10}$ to $2\times 10^{-5}$ over the first 3\% of training, then annealed to zero with a cosine schedule. We apply gradient clipping with threshold 1 and use no weight decay.

\paragraph{Results.}
\Cref{fig:gemma_smoe_deepseek} (Right) compares training loss across methods. Naive master-weight removal diverges under both round-to-nearest (RTN) and stochastic rounding (SR), whereas \methodname{} matches the master-weight baseline in both cases. In addition, \Cref{tab:benchmarks} reports zero-shot accuracy on standard benchmarks, where \methodname{} similarly recovers the performance of the master-weight models.

\begin{table}[t]
\centering
\small
\setlength{\tabcolsep}{6pt}
\caption{Fine-tuned DeepSeek-MoE-16B zero-shot benchmarks. We omit naive master weight removal baselines because training diverged in those settings. \methodname{} matches the master-weight baselines, demonstrating lossless accuracy while requiring significantly less memory.}

\scalebox{1}{
\begin{tabular}{lcccccc}
\toprule
\textbf{Method} & \textbf{ARC-C} & \textbf{ARC-E} & \textbf{GSM8K} & \textbf{HellaSwag} & \textbf{PIQA} & \textbf{MMLU} \\
\midrule
Base  & 47.53 & \textbf{73.06} & 16.15 & 77.34 & 80.36 & 37.64 \\
\midrule
INT4 w/ MW + RTN  & 48.29 & 71.38 & \textbf{16.68} & 78.76 & 80.69 & 37.87 \\
INT4 w/ MW + SR   & 48.55 & 71.13 & 16.15 & 78.78 & 80.90 & 38.57 \\
\midrule
INT4 w/o MW \methodname{} + RTN & \textbf{49.15} & 71.59 & 16.30 & \textbf{78.88} & 81.34 & \textbf{38.63} \\
INT4 w/o MW \methodname{} + SR & 48.55 & 71.17 & 16.00 & 78.84 & \textbf{81.50} & 38.41\\
\bottomrule
\end{tabular}
}
\label{tab:benchmarks}
\end{table}
\section{Conclusion}
\methodname{} is the first general-purpose, scalable method for quantized LLM training without master weights. It removes high-precision accumulation by forming an error-feedback loop through the optimizer’s momentum, with no additional memory overhead.
Our analysis shows that \methodname{} avoids the instability of naive master-weight removal. Empirically, across dense Transformers and SMoE models, \methodname{} nearly matches high-precision baselines while improving the static-memory versus loss trade-off, showing that it can serve as a practical building block for future low-precision training.
\paragraph{Limitations.}
Both theory and experiments indicate that \methodname{} performs best with stochastic rounding (SR). While SR is becoming more common in hardware, some devices only support round-to-nearest (RTN). In that setting, \methodname{} still outperforms naive approaches but can exhibit a higher noise floor, consistent with our theory. Moreover, when master weights are available, RTN generally slightly outperforms SR in practice \citep{quartet}; in contrast, \methodname{} relies on the unbiasedness of SR for its strongest guarantees. This introduces a slight accuracy ceiling relative to the best RTN-based master-weight baselines.

\bibliography{main}

@inproceedings{quest,
  title={QuEST: Training Accurate LLMs over Highly-Compressed Weights and Activation},
  author={Panferov, Andrei and Chen, Jiale and Tabesh, Soroush and Castro, Roberto L and Nikdan, Mahdi and Alistarh, Dan},
  booktitle={Sparsity in LLMs (SLLM): Deep Dive into Mixture of Experts, Quantization, Hardware, and Inference}
}

@article{chinchilla,
  title={Training compute-optimal large language models},
  author={Hoffmann, Jordan and Borgeaud, Sebastian and Mensch, Arthur and Buchatskaya, Elena and Cai, Trevor and Rutherford, Eliza and Casas, Diego de Las and Hendricks, Lisa Anne and Welbl, Johannes and Clark, Aidan and others},
  journal={arXiv preprint arXiv:2203.15556},
  year={2022}
}

@article{lm1b,
  title={One billion word benchmark for measuring progress in statistical language modeling},
  author={Chelba, Ciprian and Mikolov, Tomas and Schuster, Mike and Ge, Qi and Brants, Thorsten and Koehn, Phillipp and Robinson, Tony},
  journal={arXiv preprint arXiv:1312.3005},
  year={2013}
}

@article{amp,
  title={Mixed precision training},
  author={Micikevicius, Paulius and Narang, Sharan and Alben, Jonah and Diamos, Gregory and Elsen, Erich and Garcia, David and Ginsburg, Boris and Houston, Michael and Kuchaiev, Oleksii and Venkatesh, Ganesh and others},
  journal={arXiv preprint arXiv:1710.03740},
  year={2017}
}

@article{fp8lm,
  title={Fp8-lm: Training fp8 large language models},
  author={Peng, Houwen and Wu, Kan and Wei, Yixuan and Zhao, Guoshuai and Yang, Yuxiang and Liu, Ze and Xiong, Yifan and Yang, Ziyue and Ni, Bolin and Hu, Jingcheng and others},
  journal={arXiv preprint arXiv:2310.18313},
  year={2023}
}

@article{deepseekv3,
  title={Deepseek-v3 technical report},
  author={Liu, Aixin and Feng, Bei and Xue, Bing and Wang, Bingxuan and Wu, Bochao and Lu, Chengda and Zhao, Chenggang and Deng, Chengqi and Zhang, Chenyu and Ruan, Chong and others},
  journal={arXiv preprint arXiv:2412.19437},
  year={2024}
}

@article{t5,
  title={Exploring the limits of transfer learning with a unified text-to-text transformer},
  author={Raffel, Colin and Shazeer, Noam and Roberts, Adam and Lee, Katherine and Narang, Sharan and Matena, Michael and Zhou, Yanqi and Li, Wei and Liu, Peter J},
  journal={Journal of machine learning research},
  volume={21},
  number={140},
  pages={1--67},
  year={2020}
}

@article{sentencepiece,
  title={SentencePiece: A simple and language independent subword tokenizer and detokenizer for neural text processing},
  author={Kudo, Taku and Richardson, John},
  journal={arXiv preprint arXiv:1808.06226},
  year={2018}
}

@article{gemma3,
  title={Gemma 3 technical report},
  author={Gemma, Team and Kamath, Aishwarya and Ferret, Johan and Pathak, Shreya and Vieillard, Nino and Merhej, Ramona and Perrin, Sarah and Matejovicova, Tatiana and Ram{\'e}, Alexandre and Rivi{\`e}re, Morgane and others},
  journal={arXiv preprint arXiv:2503.19786},
  year={2025}
}

@article{ef21,
  title={EF21: A new, simpler, theoretically better, and practically faster error feedback},
  author={Richt{\'a}rik, Peter and Sokolov, Igor and Fatkhullin, Ilyas},
  journal={Advances in Neural Information Processing Systems},
  volume={34},
  pages={4384--4396},
  year={2021}
}

@article{kimik2,
  title={Kimi k2: Open agentic intelligence},
  author={Team, Kimi and Bai, Yifan and Bao, Yiping and Chen, Guanduo and Chen, Jiahao and Chen, Ningxin and Chen, Ruijue and Chen, Yanru and Chen, Yuankun and Chen, Yutian and others},
  journal={arXiv preprint arXiv:2507.20534},
  year={2025}
}

@article{gptq,
  title={Gptq: Accurate post-training quantization for generative pre-trained transformers},
  author={Frantar, Elias and Ashkboos, Saleh and Hoefler, Torsten and Alistarh, Dan},
  journal={arXiv preprint arXiv:2210.17323},
  year={2022}
}

@article{awq,
  title={Awq: Activation-aware weight quantization for on-device llm compression and acceleration},
  author={Lin, Ji and Tang, Jiaming and Tang, Haotian and Yang, Shang and Chen, Wei-Ming and Wang, Wei-Chen and Xiao, Guangxuan and Dang, Xingyu and Gan, Chuang and Han, Song},
  journal={Proceedings of machine learning and systems},
  volume={6},
  pages={87--100},
  year={2024}
}

@article{lsq,
  title={Learned step size quantization},
  author={Esser, Steven K and McKinstry, Jeffrey L and Bablani, Deepika and Appuswamy, Rathinakumar and Modha, Dharmendra S},
  journal={arXiv preprint arXiv:1902.08153},
  year={2019}
}

@article{pact,
  title={Pact: Parameterized clipping activation for quantized neural networks},
  author={Choi, Jungwook and Wang, Zhuo and Venkataramani, Swagath and Chuang, Pierce I-Jen and Srinivasan, Vijayalakshmi and Gopalakrishnan, Kailash},
  journal={arXiv preprint arXiv:1805.06085},
  year={2018}
}

@article{dorefa,
  title={Dorefa-net: Training low bitwidth convolutional neural networks with low bitwidth gradients},
  author={Zhou, Shuchang and Wu, Yuxin and Ni, Zekun and Zhou, Xinyu and Wen, He and Zou, Yuheng},
  journal={arXiv preprint arXiv:1606.06160},
  year={2016}
}

@inproceedings{qil,
  title={Learning to quantize deep networks by optimizing quantization intervals with task loss},
  author={Jung, Sangil and Son, Changyong and Lee, Seohyung and Son, Jinwoo and Han, Jae-Joon and Kwak, Youngjun and Hwang, Sung Ju and Choi, Changkyu},
  booktitle={Proceedings of the IEEE/CVF conference on computer vision and pattern recognition},
  pages={4350--4359},
  year={2019}
}

@article{bitnet,
  title={Bitnet: Scaling 1-bit transformers for large language models},
  author={Wang, Hongyu and Ma, Shuming and Dong, Li and Huang, Shaohan and Wang, Huaijie and Ma, Lingxiao and Yang, Fan and Wang, Ruiping and Wu, Yi and Wei, Furu},
  journal={arXiv preprint arXiv:2310.11453},
  year={2023}
}

@inproceedings{effqat,
  title={Efficientqat: Efficient quantization-aware training for large language models},
  author={Chen, Mengzhao and Shao, Wenqi and Xu, Peng and Wang, Jiahao and Gao, Peng and Zhang, Kaipeng and Luo, Ping},
  booktitle={Proceedings of the 63rd Annual Meeting of the Association for Computational Linguistics (Volume 1: Long Papers)},
  pages={10081--10100},
  year={2025}
}

@article{halo,
  title={HALO: Hadamard-Assisted Lower-Precision Optimization for LLMs},
  author={Ashkboos, Saleh and Nikdan, Mahdi and Tabesh, Soroush and Castro, Roberto L and Hoefler, Torsten and Alistarh, Dan},
  journal={arXiv preprint arXiv:2501.02625},
  year={2025}
}

@article{quartet,
  title={Quartet: Native FP4 Training Can Be Optimal for Large Language Models},
  author={Castro, Roberto L and Panferov, Andrei and Tabesh, Soroush and Sieberling, Oliver and Chen, Jiale and Nikdan, Mahdi and Ashkboos, Saleh and Alistarh, Dan},
  journal={arXiv preprint arXiv:2505.14669},
  year={2025}
}

@article{alberttseng,
  title={Training llms with mxfp4},
  author={Tseng, Albert and Yu, Tao and Park, Youngsuk},
  journal={arXiv preprint arXiv:2502.20586},
  year={2025}
}

@article{fp4alltheway,
  title={FP4 All the Way: Fully Quantized Training of LLMs},
  author={Chmiel, Brian and Fishman, Maxim and Banner, Ron and Soudry, Daniel},
  journal={arXiv preprint arXiv:2505.19115},
  year={2025}
}

@article{quarot,
  title={Quarot: Outlier-free 4-bit inference in rotated llms},
  author={Ashkboos, Saleh and Mohtashami, Amirkeivan and Croci, Maximilian L and Li, Bo and Cameron, Pashmina and Jaggi, Martin and Alistarh, Dan and Hoefler, Torsten and Hensman, James},
  journal={Advances in Neural Information Processing Systems},
  volume={37},
  pages={100213--100240},
  year={2024}
}

@article{spinquant,
  title={Spinquant: Llm quantization with learned rotations},
  author={Liu, Zechun and Zhao, Changsheng and Fedorov, Igor and Soran, Bilge and Choudhary, Dhruv and Krishnamoorthi, Raghuraman and Chandra, Vikas and Tian, Yuandong and Blankevoort, Tijmen},
  journal={arXiv preprint arXiv:2405.16406},
  year={2024}
}

@article{flatquant,
  title={Flatquant: Flatness matters for llm quantization},
  author={Sun, Yuxuan and Liu, Ruikang and Bai, Haoli and Bao, Han and Zhao, Kang and Li, Yuening and Hu, Jiaxin and Yu, Xianzhi and Hou, Lu and Yuan, Chun and others},
  journal={arXiv preprint arXiv:2410.09426},
  year={2024}
}

@inproceedings{llmqat,
  title={Llm-qat: Data-free quantization aware training for large language models},
  author={Liu, Zechun and Oguz, Barlas and Zhao, Changsheng and Chang, Ernie and Stock, Pierre and Mehdad, Yashar and Shi, Yangyang and Krishnamoorthi, Raghuraman and Chandra, Vikas},
  booktitle={Findings of the Association for Computational Linguistics: ACL 2024},
  pages={467--484},
  year={2024}
}

@article{ondevice,
  title={On-device training under 256kb memory},
  author={Lin, Ji and Zhu, Ligeng and Chen, Wei-Ming and Wang, Wei-Chen and Gan, Chuang and Han, Song},
  journal={Advances in Neural Information Processing Systems},
  volume={35},
  pages={22941--22954},
  year={2022}
}

@article{collage,
  title={Collage: Light-weight low-precision strategy for LLM training},
  author={Yu, Tao and Gupta, Gaurav and Gopalswamy, Karthick and Mamidala, Amith and Zhou, Hao and Huynh, Jeffrey and Park, Youngsuk and Diamant, Ron and Deoras, Anoop and Huan, Luke},
  journal={arXiv preprint arXiv:2405.03637},
  year={2024}
}

@article{apt,
  title={APT: The master-copy-free training method for quantised neural network on edge devices},
  author={Huang, Tian and Luo, Tao and Zhou, Joey Tianyi},
  journal={Journal of Parallel and Distributed Computing},
  volume={166},
  pages={95--103},
  year={2022},
  publisher={Elsevier}
}

@article{opt8bit,
  title={8-bit optimizers via block-wise quantization},
  author={Dettmers, Tim and Lewis, Mike and Shleifer, Sam and Zettlemoyer, Luke},
  journal={arXiv preprint arXiv:2110.02861},
  year={2021}
}

@article{coat,
  title={Coat: Compressing optimizer states and activation for memory-efficient fp8 training},
  author={Xi, Haocheng and Cai, Han and Zhu, Ligeng and Lu, Yao and Keutzer, Kurt and Chen, Jianfei and Han, Song},
  journal={arXiv preprint arXiv:2410.19313},
  year={2024}
}

@article{bf16sr,
  title={Stochastic rounding for LLM training: Theory and practice},
  author={Ozkara, Kaan and Yu, Tao and Park, Youngsuk},
  journal={arXiv preprint arXiv:2502.20566},
  year={2025}
}

@article{scalingfp8,
  title={Scaling fp8 training to trillion-token llms},
  author={Fishman, Maxim and Chmiel, Brian and Banner, Ron and Soudry, Daniel},
  journal={arXiv preprint arXiv:2409.12517},
  year={2024}
}

@article{opt4bit,
  title={Memory efficient optimizers with 4-bit states},
  author={Li, Bingrui and Chen, Jianfei and Zhu, Jun},
  journal={Advances in Neural Information Processing Systems},
  volume={36},
  pages={15136--15171},
  year={2023}
}

@inproceedings{onebitsgd,
  title={1-bit stochastic gradient descent and its application to data-parallel distributed training of speech DNNs.},
  author={Seide, Frank and Fu, Hao and Droppo, Jasha and Li, Gang and Yu, Dong},
  booktitle={Interspeech},
  volume={2014},
  pages={1058--1062},
  year={2014},
  organization={Singapore}
}

@inproceedings{onebitadam,
  title={1-bit adam: Communication efficient large-scale training with adam’s convergence speed},
  author={Tang, Hanlin and Gan, Shaoduo and Awan, Ammar Ahmad and Rajbhandari, Samyam and Li, Conglong and Lian, Xiangru and Liu, Ji and Zhang, Ce and He, Yuxiong},
  booktitle={International Conference on Machine Learning},
  pages={10118--10129},
  year={2021},
  organization={PMLR}
}

@article{zeropp,
  title={Zero++: Extremely efficient collective communication for giant model training},
  author={Wang, Guanhua and Qin, Heyang and Jacobs, Sam Ade and Holmes, Connor and Rajbhandari, Samyam and Ruwase, Olatunji and Yan, Feng and Yang, Lei and He, Yuxiong},
  journal={arXiv preprint arXiv:2306.10209},
  year={2023}
}

@article{elmo,
  title={ELMO: Efficiency via Low-precision and Peak Memory Optimization in Large Output Spaces},
  author={Zhang, Jinbin and Ullah, Nasib and Schultheis, Erik and Babbar, Rohit},
  journal={arXiv preprint arXiv:2510.11168},
  year={2025}
}

@article{jetfire,
  title={Jetfire: Efficient and accurate transformer pretraining with int8 data flow and per-block quantization},
  author={Xi, Haocheng and Chen, Yuxiang and Zhao, Kang and Teh, Kai Jun and Chen, Jianfei and Zhu, Jun},
  journal={arXiv preprint arXiv:2403.12422},
  year={2024}
}

@article{mx,
  title={Microscaling data formats for deep learning},
  author={Rouhani, Bita Darvish and Zhao, Ritchie and More, Ankit and Hall, Mathew and Khodamoradi, Alireza and Deng, Summer and Choudhary, Dhruv and Cornea, Marius and Dellinger, Eric and Denolf, Kristof and others},
  journal={arXiv preprint arXiv:2310.10537},
  year={2023}
}

@article{deepseekmoe,
  author={Damai Dai and Chengqi Deng and Chenggang Zhao and R. X. Xu and Huazuo Gao and Deli Chen and Jiashi Li and Wangding Zeng and Xingkai Yu and Y. Wu and Zhenda Xie and Y. K. Li and Panpan Huang and Fuli Luo and Chong Ruan and Zhifang Sui and Wenfeng Liang},
  title={DeepSeekMoE: Towards Ultimate Expert Specialization in Mixture-of-Experts Language Models}, 
  journal   = {CoRR},
  volume    = {abs/2401.06066},
  year      = {2024},
  url       = {https://arxiv.org/abs/2401.06066},
}

@article{qlora,
  title={Qlora: Efficient finetuning of quantized llms},
  author={Dettmers, Tim and Pagnoni, Artidoro and Holtzman, Ari and Zettlemoyer, Luke},
  journal={Advances in neural information processing systems},
  volume={36},
  pages={10088--10115},
  year={2023}
}
\appendix
\newpage

\crefalias{section}{appsec}
\section{Exact Error Injection}
\label{apx:exact-error-injection}

This appendix shows that SGDM with high-precision master weights can be reproduced \emph{exactly} using only quantized weights, provided the momentum buffer receives an ``ideal'' correction that depends on both the current and previous quantization residuals.

\paragraph{SGDM with Master Weights.}
Let $\qfn(\cdot)$ be the weight quantizer. Let $\vtheta_t$ denote the high-precision master weights, and let
$\hat{\vtheta}^{\text{MW}}_t \leftarrow \qfn(\vtheta_t)$
be the quantized weights used for the forward/backward pass at step $t$.
Using the gradient
\begin{equation}
\vg_t \leftarrow \nabla f(\hat{\vtheta}^{\text{MW}}_t),
\label{eq:mw_grad}
\end{equation}
SGDM with master weights updates
\begin{align}
\vm^{\text{MW}}_{t+1} &\leftarrow \beta\,\vm^{\text{MW}}_t + (1-\beta)\,\vg_t,
\label{eq:mw_mom}\\
\vtheta_{t+1} &\leftarrow \vtheta_t - \eta\,\vm^{\text{MW}}_{t+1},
\label{eq:mw_theta}\\
\hat{\vtheta}^{\text{MW}}_{t+1} &\leftarrow \qfn(\vtheta_{t+1}),
\label{eq:mw_quant}
\end{align}
and we define the quantization residual of $\vtheta_{t+1}$ as
\begin{equation}
\ve^{\text{MW}}_{t+1} \;:=\; \vtheta_{t+1} - \hat{\vtheta}^{\text{MW}}_{t+1}.
\label{eq:mw_err}
\end{equation}

\paragraph{No-Master-Weight SGDM with Ideal Momentum Injection.}
This variant stores only quantized weights $\hat{\vtheta}^{\text{IM}}_t$, a momentum buffer $\vm^{\text{IM}}_t$, and the previous residual $\ve^{\text{IM}}_t$. At step $t$, it computes
\begin{align}
\vg_t &\leftarrow \nabla f(\hat{\vtheta}^{\text{IM}}_t), \nonumber\\
\bar{\vm}_{t+1} &\leftarrow \beta\,\vm^{\text{IM}}_t + (1-\beta)\,\vg_t, \label{eq:im_bar_m}\\
\tilde{\vtheta}_{t+1} &\leftarrow \hat{\vtheta}^{\text{IM}}_t - \eta\,\bar{\vm}_{t+1}, \label{eq:im_tilde_theta}\\
\hat{\vtheta}^{\text{IM}}_{t+1} &\leftarrow \qfn(\tilde{\vtheta}_{t+1}), \label{eq:im_quant}\\
\ve^{\text{IM}}_{t+1} &\leftarrow \tilde{\vtheta}_{t+1} - \hat{\vtheta}^{\text{IM}}_{t+1}, \label{eq:im_err}\\
\vm^{\text{IM}}_{t+1} &\leftarrow \bar{\vm}_{t+1} + \frac{1}{\eta}\ve^{\text{IM}}_{t} - \frac{1}{\eta\beta}\ve^{\text{IM}}_{t+1}. \label{eq:im_inject}
\end{align}

\paragraph{Theorem (Exact equivalence).}
Assume SGDM with master weights starts from $(\vtheta_0,\vm^{\text{MW}}_0)$. Initialize the injected method by
\begin{equation}
\hat{\vtheta}^{\text{IM}}_0 \leftarrow \qfn(\vtheta_0), 
\qquad
\ve^{\text{IM}}_{0} \leftarrow \vtheta_0 - \hat{\vtheta}^{\text{IM}}_0,
\qquad
\vm^{\text{IM}}_0 \leftarrow \vm^{\text{MW}}_0 - \frac{1}{\eta\beta}\ve^{\text{IM}}_{0}.
\label{eq:init_im}
\end{equation}
Then, for all $t\ge 0$, the quantized iterates produced by the injected method satisfy
\[
\hat{\vtheta}^{\text{IM}}_t \;=\; \hat{\vtheta}^{\text{MW}}_t,
\]
and therefore the two procedures produce identical gradients at every step.

\paragraph{Proof.}
Define the \emph{implicit} master weights and momentum corresponding to the injected method by
\begin{equation}
\vtheta^{\star}_t \;:=\; \hat{\vtheta}^{\text{IM}}_t + \ve^{\text{IM}}_{t},
\qquad
\vm^{\star}_t \;:=\; \vm^{\text{IM}}_t + \frac{1}{\eta\beta}\ve^{\text{IM}}_{t}.
\label{eq:implicit_defs}
\end{equation}
By \eqref{eq:init_im}, we have $\vtheta^{\star}_0=\vtheta_0$ and $\vm^{\star}_0=\vm^{\text{MW}}_0$.

From \eqref{eq:im_err}, we have
\begin{equation}
\tilde{\vtheta}_{t+1} \;=\; \hat{\vtheta}^{\text{IM}}_{t+1} + \ve^{\text{IM}}_{t+1}.
\label{eq:tilde_decomp}
\end{equation}
Hence,
\begin{equation}
\vtheta^{\star}_{t+1}
\;=\;
\hat{\vtheta}^{\text{IM}}_{t+1} + \ve^{\text{IM}}_{t+1}
\;=\;
\tilde{\vtheta}_{t+1}.
\label{eq:theta_star_is_tilde}
\end{equation}
Combining with \eqref{eq:im_quant}, we get
$\hat{\vtheta}^{\text{IM}}_{t+1} \leftarrow \qfn(\vtheta^\star_{t+1})$.

Next, using \eqref{eq:im_inject} and \eqref{eq:implicit_defs},
\begin{align}
\vm^\star_{t+1}
&=
\vm^{\text{IM}}_{t+1} + \frac{1}{\eta\beta}\ve^{\text{IM}}_{t+1} \nonumber\\
&=
\bar{\vm}_{t+1} + \frac{1}{\eta}\ve^{\text{IM}}_{t}
- \frac{1}{\eta\beta}\ve^{\text{IM}}_{t+1}
+ \frac{1}{\eta\beta}\ve^{\text{IM}}_{t+1} \nonumber\\
&=
\bar{\vm}_{t+1} + \frac{1}{\eta}\ve^{\text{IM}}_{t} \nonumber\\
&=
\beta\,\vm^{\text{IM}}_t + (1-\beta)\,\vg_t + \frac{1}{\eta}\ve^{\text{IM}}_{t} \label{eq:mstar_mid}\\
&=
\beta\Bigl(\vm^{\text{IM}}_t + \frac{1}{\eta\beta}\ve^{\text{IM}}_{t}\Bigr) + (1-\beta)\,\vg_t \nonumber\\
&=
\beta\,\vm^\star_t + (1-\beta)\,\vg_t.
\label{eq:mstar_rec}
\end{align}
Thus $\vm^\star_{t+1}$ follows the same SGDM momentum recurrence as \eqref{eq:mw_mom}.

Finally, using \eqref{eq:theta_star_is_tilde}, \eqref{eq:im_tilde_theta}, and \eqref{eq:mstar_mid},
\begin{align}
\vtheta^\star_{t+1}
&=
\tilde{\vtheta}_{t+1}
=
\hat{\vtheta}^{\text{IM}}_t - \eta\,\bar{\vm}_{t+1} \nonumber\\
&=
(\hat{\vtheta}^{\text{IM}}_t + \ve^{\text{IM}}_{t})
- \eta\bigl(\bar{\vm}_{t+1} + \tfrac{1}{\eta}\ve^{\text{IM}}_{t}\bigr) \nonumber\\
&=
\vtheta^\star_t - \eta\,\vm^\star_{t+1},
\end{align}
which matches the master-weight update \eqref{eq:mw_theta}. Therefore, with identical initial conditions, the implicit variables $(\vtheta^\star_t,\vm^\star_t)$ evolve exactly as SGDM with master weights, implying
\[
\hat{\vtheta}^{\text{IM}}_t
\;=\;
\qfn(\vtheta^\star_t)
\;=\;
\qfn(\vtheta_t)
\;=\;
\hat{\vtheta}^{\text{MW}}_t
\qquad \text{for all } t.
\]
\hfill$\square$

\section{Convergence Proofs}
\label{apx:convergence-proofs}
\subsection{Proof of Lemma \ref{lemma:virtual_dynamics}}
\label{proof:lemma_virtual_dynamics}
\begin{proof}
First, substitute $\tilde{\vm}_{t+1}$ from Eq.~\eqref{eq:m_hat} into Eq.~\eqref{eq:theta_tilde}:
\begin{equation}
    \tilde{\vtheta}_{t+1} = \hat{\vtheta}_t - \eta (\hat{\vm}_{t+1} - \alpha \ve_{t+1}).
\end{equation}
Using $\ve_{t+1} = \tilde{\vtheta}_{t+1} - \hat{\vtheta}_{t+1}$, we rearrange to solve for $\hat{\vtheta}_{t+1}$:
\begin{align}
    \hat{\vtheta}_{t+1} + \ve_{t+1} &= \hat{\vtheta}_t - \eta \hat{\vm}_{t+1} + \eta \alpha \ve_{t+1} \nonumber \\
    \hat{\vtheta}_{t+1} &= \hat{\vtheta}_t - \eta \hat{\vm}_{t+1} - (1 - \eta \alpha) \ve_{t+1}.
\end{align}
Substituting $\alpha = \frac{1}{\eta}(1 - \frac{1}{\beta})$, we have $1 - \eta \alpha = \frac{1}{\beta}$. Thus:
\begin{equation} \label{eq:theta_hat_update_app}
    \hat{\vtheta}_{t+1} = \hat{\vtheta}_t - \eta \hat{\vm}_{t+1} - \frac{1}{\beta} \ve_{t+1}.
\end{equation}
Now, examine the update of the virtual sequence $\vtheta_{t+1}$:
\begin{equation}
    \vtheta_{t+1} = \hat{\vtheta}_{t+1} - \frac{\eta \beta}{1-\beta} \hat{\vm}_{t+1}.
\end{equation}
We expand this expression:
\begin{align}
    \vtheta_{t+1} &= \hat{\vtheta}_{t+1} - \frac{\eta \beta}{1-\beta} \hat{\vm}_{t+1} \nonumber \\
    &= (\tilde{\vtheta}_{t+1} - \ve_{t+1}) - \frac{\eta \beta}{1-\beta} (\tilde{\vm}_{t+1} + \alpha \ve_{t+1}) \nonumber \\
    &= (\hat{\vtheta}_t - \eta \tilde{\vm}_{t+1} - \ve_{t+1}) - \frac{\eta \beta}{1-\beta} \tilde{\vm}_{t+1} - \frac{\eta \beta \alpha}{1-\beta} \ve_{t+1} \nonumber \\
    &= \hat{\vtheta}_t - \eta \left( 1 + \frac{\beta}{1-\beta} \right) \tilde{\vm}_{t+1} - \left( 1 + \frac{\eta \beta \alpha}{1-\beta} \right) \ve_{t+1}.
\end{align}
Using the identity $1 + \frac{\beta}{1-\beta} = \frac{1}{1-\beta}$, the coefficient of $\tilde{\vm}_{t+1}$ is $-\frac{\eta}{1-\beta}$.
Now check the coefficient of $\ve_{t+1}$. Using $\alpha = \frac{\beta-1}{\eta \beta} = -\frac{1-\beta}{\eta \beta}$:
\begin{equation}
    1 + \frac{\eta \beta}{1-\beta} \left( -\frac{1-\beta}{\eta \beta} \right) = 1 - 1 = 0.
\end{equation}
The error term $\ve_{t+1}$ vanishes perfectly. We are left with:
\begin{equation}
    \vtheta_{t+1} = \hat{\vtheta}_t - \frac{\eta}{1-\beta} \tilde{\vm}_{t+1}.
\end{equation}
Expanding $\tilde{\vm}_{t+1} = \beta \hat{\vm}_t + (1-\beta) \nabla f(\hat{\vtheta}_t)$:
\begin{align}
    \vtheta_{t+1} &= \hat{\vtheta}_t - \frac{\eta \beta}{1-\beta} \hat{\vm}_t - \eta \nabla f(\hat{\vtheta}_t) \nonumber \\
    &= \vtheta_t - \eta \nabla f(\hat{\vtheta}_t).
\end{align}
\end{proof}

\subsection{Proof of Lemma \ref{lemma:descent}}
\label{proof:lemma_descent}
\begin{proof}
Applying the standard descent lemma to the virtual sequence $\vtheta_t$:
\begin{align}
    f(\vtheta_{t+1}) &\le f(\vtheta_t) + \inner{\nabla f(\vtheta_t)}{\vtheta_{t+1} - \vtheta_t} + \frac{L}{2} \norm{\vtheta_{t+1} - \vtheta_t}^2 \nonumber \\
    &\le f(\vtheta_t) - \eta \inner{\nabla f(\vtheta_t)}{\nabla f(\hat{\vtheta}_t)} + \frac{L \eta^2}{2} \norm{\nabla f(\hat{\vtheta}_t)}^2.
\end{align}
Using the identity $- \inner{a}{b} = -\frac{1}{2}\norm{a}^2 - \frac{1}{2}\norm{b}^2 + \frac{1}{2}\norm{a-b}^2$:
\begin{align}
    f(\vtheta_{t+1}) &\le f(\vtheta_t) - \frac{\eta}{2} \norm{\nabla f(\vtheta_t)}^2 - \frac{\eta}{2} (1 - L \eta) \norm{\nabla f(\hat{\vtheta}_t)}^2 + \frac{\eta}{2} \norm{\nabla f(\vtheta_t) - \nabla f(\hat{\vtheta}_t)}^2.
\end{align}
The term with $\norm{\nabla f(\vtheta_t)}^2$ is non-positive. Additionally, as $\eta \le \frac{1}{2L}$, we have $1 - L \eta \ge \frac{1}{2}$. Using $L$-smoothness on the last term:
\begin{equation}
    f(\vtheta_{t+1}) \le f(\vtheta_t) - \frac{\eta}{4} \norm{\nabla f(\hat{\vtheta}_t)}^2 + \frac{\eta L^2}{2} \norm{\vtheta_t - \hat{\vtheta}_t}^2.
\end{equation}
The difference between the virtual and actual (quantized) parameters is:
\begin{equation}
    \vtheta_t - \hat{\vtheta}_t = -\frac{\eta \beta}{1-\beta} \hat{\vm}_t.
\end{equation}
Substituting $C = \frac{\eta \beta}{1-\beta}$ yields $\norm{\vtheta_t - \hat{\vtheta}_t}^2 = C^2 \norm{\hat{\vm}_t}^2$.
Substituting into the inequality gives:
\begin{equation}
    f(\vtheta_{t+1}) \le f(\vtheta_t) - \frac{\eta}{4} \norm{\nabla f(\hat{\vtheta}_t)}^2 + \frac{\eta L^2 C^2}{2} \norm{\hat{\vm}_t}^2.
\end{equation}
\end{proof}

\subsection{Proof of Lemma \ref{lemma:momentum_bound}}
\label{proof:lemma_momentum_bound_stoch}

\begin{proof}
We expand the recursion for $\hat{\vm}_t$ starting from $\hat{\vm}_0 = 0$. With the updated update rule $\tilde{\vm}_{t+1} = \beta \hat{\vm}_t + (1-\beta) \nabla f(\hat{\vtheta}_t)$, the expansion becomes:
\begin{equation}
    \hat{\vm}_t = \sum_{k=1}^{t} \beta^{t-k} \left( (1-\beta)\nabla f(\hat{\vtheta}_{k-1}) + \alpha \ve_k \right).
\end{equation}
We define two components, the gradient accumulation $S_1$ and the error accumulation $S_2$:
\begin{equation}
    S_1 = (1-\beta) \sum_{k=1}^{t} \beta^{t-k} \nabla f(\hat{\vtheta}_{k-1}), \quad S_2 = \sum_{k=1}^{t} \beta^{t-k} \alpha \ve_k.
\end{equation}
Using the inequality $\|a+b\|^2 \le 2\|a\|^2 + 2\|b\|^2$, we have $\E[\|\hat{\vm}_t\|^2] \le 2\E[\|S_1\|^2] + 2\E[\|S_2\|^2]$.

For the gradient term $S_1$, we use the deterministic triangle inequality bound. The $(1-\beta)$ factor scales the sum:
\begin{equation}
    \|S_1\| \le (1-\beta) \sum_{k=1}^{t} \beta^{t-k} \|\nabla f(\hat{\vtheta}_{k-1})\| \le G (1-\beta) \sum_{j=0}^{t-1} \beta^j.
\end{equation}
Using the geometric series sum bound $\sum_{j=0}^{t-1} \beta^j \le \frac{1}{1-\beta}$, the terms cancel nicely:
\begin{equation}
    \|S_1\| \le G (1-\beta) \frac{1}{1-\beta} = G.
\end{equation}
Thus $\E[\|S_1\|^2] \le G^2$.

For the error term $S_2$, utilizing the unbiasedness assumption where $\E[\ve_k | \ve_{j}] = 0$ for $k > j$:
\begin{align}
    \E[\|S_2\|^2] &= \E\left[ \left\| \sum_{k=1}^{t} \beta^{t-k} \alpha \ve_k \right\|^2 \right] \nonumber \\
    &= \sum_{k=1}^{t} \beta^{2(t-k)} \alpha^2 \E[\|\ve_k\|^2] + \sum_{j \neq k} \text{Cross Terms} \nonumber \\
    &= \sum_{k=1}^{t} \beta^{2(t-k)} \alpha^2 \E[\|\ve_k\|^2].
\end{align}
Using $\E[\|\ve_k\|^2] \le \sigma^2$, we bound the sum by the infinite geometric series with ratio $\beta^2$:
\begin{equation}
    \E[\|S_2\|^2] \le \alpha^2 \sigma^2 \sum_{j=0}^{\infty} (\beta^2)^j = \frac{\alpha^2 \sigma^2}{1-\beta^2}.
\end{equation}
Combining these results:
\begin{equation}
    \E[\|\hat{\vm}_t\|^2] \le 2G^2 + \frac{2\alpha^2 \sigma^2}{1-\beta^2}.
\end{equation}
\end{proof}

\subsection{Proof of Theorem \ref{thm:convergence}}
\label{proof:thm_convergence}

\begin{proof}
We take expectations from both sides of the descent Lemma \ref{lemma:descent} and substitute the momentum bound $M^2$ (Lemma \ref{lemma:momentum_bound}).
\begin{equation}
    \E \left[ f(\vtheta_{t+1}) \right] \le \E \left[f(\vtheta_t)\right] - \frac{\eta}{4} \E \left[\|\nabla f(\hat{\vtheta_t})\|^2\right] + \frac{\eta L^2 C^2}{2} M^2.
\end{equation}

Rearranging to isolate the gradient norm:
\begin{equation}
    \frac{\eta}{4} \E \left[\|\nabla f(\hat{\vtheta_t})\|^2\right] \le \E \left[f(\vtheta_t) - f(\vtheta_{t+1})\right] + \frac{\eta L^2 C^2}{2} M^2.
\end{equation}

Summing from $t=0$ to $T-1$:
\begin{align}
    \frac{\eta}{4} \sum_{t=0}^{T-1} \E \left[\|\nabla f(\hat{\vtheta_t})\|^2\right] &\le \E \left[f(\vtheta_0) - f(\vtheta_T)\right] + \sum_{t=0}^{T-1} \frac{\eta L^2 C^2}{2} M^2 \nonumber \\
    &\le f(\vtheta_0) - f^* + T \frac{\eta L^2 C^2}{2} M^2.
\end{align}

Dividing by $T \eta / 4$:
\begin{equation}
    \frac{1}{T} \sum_{t=0}^{T-1} \E \left[\|\nabla f(\hat{\vtheta_t})\|^2 \right] \le \frac{4(f(\vtheta_0) - f^*)}{\eta T} + 2L^2 C^2 M^2.
\end{equation}
Defining $\sigma_{\text{quant}}^2 = 2L^2 C^2 M^2$ yields the final result.
\end{proof}

\subsection{Proof of Lemma \ref{lemma:momentum_bound_det}}
\label{proof:lemma_momentum_bound_det}
\begin{proof}
Let $\|\ve_t\| \le \delta$ (absolute error bound).
\begin{equation}
    \|\hat{\vm}_{t+1}\| \le \beta \|\hat{\vm}_t\| + (1-\beta)\|\nabla f(\hat{\vtheta}_t)\| + |\alpha| \|\ve_{t+1}\|.
\end{equation}
Using the bounded gradient assumption $\|\nabla f(\vtheta)\| \le G$:
\begin{equation}
    \|\hat{\vm}_{t+1}\| \le \beta \|\hat{\vm}_t\| + (1-\beta) G + |\alpha| \delta.
\end{equation}
This is a linear recurrence of the form $x_{t+1} \le \beta x_t + K$. Assuming $\hat{\vm}_0 = 0$, the sequence is bounded by the sum of the geometric series:
\begin{equation}
    \|\hat{\vm}_t\| \le \sum_{i=0}^{t} \beta^i ((1-\beta)G + |\alpha|\delta) \le G + \frac{|\alpha| \delta}{1-\beta} \coloneqq M.
\end{equation}
\end{proof}

\subsection{Proof of Theorem \ref{thm:convergence_det}}
\label{apx:convergence_proof_det}

\begin{proof}
We use Lemmas \ref{lemma:descent} and \ref{lemma:momentum_bound_det}.

Summing the descent inequality from $t=0$ to $T-1$:
\begin{equation}
    f(\vtheta_T) \le f(\vtheta_0) - \frac{\eta}{4} \sum_{t=0}^{T-1} \norm{\nabla f(\hat{\vtheta_t})}^2 + \frac{\eta T L^2 C^2}{2} M_{\text{det}}^2.
\end{equation}
Rearranging and using $f^* \le f(\vtheta_T)$:
\begin{equation}
    \frac{1}{T} \sum_{t=0}^{T-1} \norm{\nabla f(\hat{\vtheta_t})}^2 \le \frac{4(f(\vtheta_0) - f^*)}{\eta T} + 2L^2 C^2 M_{\text{det}}^2.
\end{equation}
Defining $\Gamma_{\text{quant}}^2 = 2L^2 C^2 M_{\text{det}}^2$ completes the proof.
\end{proof}

\section{Formal Analysis of the Worst-Case Lower-Bounds}
\label{apx:fundamental}

This appendix provides proofs for the claims made in \Cref{sec:fundamental}.

All three regimes considered below can be written in the linear form
\begin{equation}
\label{eq:lin_form}
\begin{aligned}
    x_{t+1} &= a x_t + b m_t + B_1 \xi_{t+1},\\
    m_{t+1} &= c x_t + d m_t + B_2 \xi_{t+1},
\end{aligned}
\end{equation}
for constants $a,b,c,d,B_1,B_2$ that depend on the regime.
Define the second moments
\begin{equation}
    u_t=\E[x_t^2],\qquad v_t=\E[x_t m_t],\qquad w_t=\E[m_t^2].
\end{equation}

\begin{lemma}[Second-moment update equations]
\label{lem:moment_recursions}
The dynamics \eqref{eq:lin_form} imply
\begin{align}
    u_{t+1} &= a^2 u_t + 2ab\, v_t + b^2 w_t + B_1^2\sigma^2, \label{eq:u_rec}\\
    v_{t+1} &= ac\, u_t + (ad+bc)\, v_t + bd\, w_t + B_1B_2\sigma^2, \label{eq:v_rec}\\
    w_{t+1} &= c^2 u_t + 2cd\, v_t + d^2 w_t + B_2^2\sigma^2. \label{eq:w_rec}
\end{align}
\end{lemma}
\begin{proof}
Expand each square/product and remove all cross terms.
For example, for $u_{t+1}$:
\[
u_{t+1}=\E[(a x_t + b m_t + B_1\xi_{t+1})^2]
= a^2 u_t + 2ab v_t + b^2 w_t + B_1^2\E[\xi_{t+1}^2],
\]
since $\E[x_t\xi_{t+1}]=\E[m_t\xi_{t+1}]=0$ and $\E[\xi_{t+1}^2]=\sigma^2$.
The proofs for $v_{t+1}$ and $w_{t+1}$ are identical.
\end{proof}

\paragraph{Stability.}
Let $A=\begin{pmatrix}a & b\\ c & d\end{pmatrix}$ denote the deterministic part of \eqref{eq:lin_form}.
A sufficient and standard condition for existence of a unique stationary second moment is $\rho(A)<1$, where $\rho(A)$ indicates $A$'s largest absolute eigenvalue.
For the SGDM parameters used below, this holds whenever
\begin{equation}
\label{eq:eta_stable}
    0<\eta<\frac{2(1+\beta)}{(1-\beta)L}.
\end{equation}
All stationary calculations below assume \eqref{eq:eta_stable}, which also guarantees that the denominators appearing
in the closed forms are strictly positive.

\subsection{Fundamental limits on $f(x)=\frac{L}{2}x^2$}

We now analyze the stationary squared gradient of the \emph{quantized} parameter used by the model.
For any regime, define the (steady-state) metric
\begin{equation}
    \mathcal L \coloneqq \lim_{t\to\infty}\E[g(\hat x_t)^2]
    = L^2 \lim_{t\to\infty}\E[\hat x_t^2],
\end{equation}
where $\hat x_t$ is the parameter seen by the forward/backward pass (quantized weights).

\subsubsection{SGDM with master weights}

\paragraph{Algorithm.}
We store a full-precision master weight $x_t$. Each step quantizes it for the gradient:
\begin{equation}
    \hat x_t=q(x_t)=x_t+\xi_t,
\end{equation}
then performs SGDM using $\hat x_t$:
\begin{equation}
    m_{t+1}=\beta m_t + (1-\beta)L\hat x_t,\qquad x_{t+1}=x_t-\eta m_{t+1}.
\end{equation}

\paragraph{Linear form.}
Let $c\coloneqq (1-\beta)L$, and define
\begin{equation}
    a\coloneqq 1-\eta c,\qquad b\coloneqq -\eta\beta,\qquad d\coloneqq \beta.
\end{equation}
Using $\hat x_t=x_t+\xi_t$, we obtain
\begin{equation}
\label{eq:mw_lin}
\begin{aligned}
    x_{t+1} &= a x_t + b m_t + (-\eta c)\,\xi_t,\\
    m_{t+1} &= c x_t + d m_t + c\,\xi_t.
\end{aligned}
\end{equation}
This matches \eqref{eq:lin_form} with $(B_1,B_2)=(-\eta c,\;c)$.

\paragraph{Stationary second moments.}
Let $(u,v,w)$ denote the stationary solution of \eqref{eq:u_rec}--\eqref{eq:w_rec}. Plugging $B_1=-\eta c$ and $B_2=c$
into Lemma~\ref{lem:moment_recursions} and setting $(u_{t+1},v_{t+1},w_{t+1})=(u,v,w)$ yields the linear system
\begin{align}
    u &= a^2 u + 2ab v + b^2 w + \eta^2 c^2\sigma^2, \label{eq:mw_u}\\
    v &= ac\,u + (ad+bc)\,v + bd\,w - \eta c^2\sigma^2, \label{eq:mw_v}\\
    w &= c^2 u + 2cd\,v + d^2 w + c^2\sigma^2. \label{eq:mw_w}
\end{align}
We solve it by elimination.

From \eqref{eq:mw_w} and $d=\beta$,
\begin{equation}
\label{eq:mw_w_solve}
    (1-\beta^2)w = c^2(u+\sigma^2) + 2c\beta v
    \quad\Longrightarrow\quad
    w=\frac{c^2(u+\sigma^2)+2c\beta v}{1-\beta^2}.
\end{equation}
Substitute \eqref{eq:mw_w_solve} into \eqref{eq:mw_v}. Using $ad+bc=\beta(1-\eta c) + (-\eta\beta)c = \beta - 2\eta\beta c$
and $bd=b\beta=-\eta\beta^2$, we rewrite \eqref{eq:mw_v} as
\begin{equation}
\label{eq:mw_v_mid}
    v = ac\,u + (\beta-2\eta\beta c)\,v - \eta\beta^2 w - \eta c^2\sigma^2.
\end{equation}
Move the $v$ and $w$ terms to the left and substitute $w$ from \eqref{eq:mw_w_solve}. This yields a single linear equation
in $v$ and $u$, which solves to
\begin{equation}
\label{eq:mw_v_solve}
    v = \frac{L^2\eta\sigma^2(\beta-1)}{\,2(1+\beta)-L\eta(1-\beta)\,}.
\end{equation}
Plugging \eqref{eq:mw_v_solve} back into \eqref{eq:mw_w_solve} gives
\begin{equation}
\label{eq:mw_w_final}
    w = \frac{2L^2\sigma^2(1-\beta)}{\,2(1+\beta)-L\eta(1-\beta)\,}.
\end{equation}
Finally, substitute \eqref{eq:mw_v_solve} and \eqref{eq:mw_w_final} into \eqref{eq:mw_u}. Solving for $u$ yields
\begin{equation}
\label{eq:mw_u_final}
    u = \E[x^2]
    = \frac{L\eta\sigma^2(1+\beta)}{\,2(1+\beta)-L\eta(1-\beta)\,}.
\end{equation}

\paragraph{Limit of the squared gradient.}
The model uses $\hat x=x+\xi$ with $\E[x\xi]=0$. Hence
\begin{equation}
    \E[\hat x^2]=\E[x^2]+\E[\xi^2]=u+\sigma^2.
\end{equation}
Therefore the stationary squared gradient satisfies
\begin{equation}
\label{eq:mw_L}
    \mathcal L_{\mathrm{MW}} = L^2(u+\sigma^2).
\end{equation}
Taking $\eta\to 0$ in \eqref{eq:mw_u_final} gives $u\to 0$, so
\begin{equation}
\label{eq:mw_limit}
    \lim_{\eta\to 0}\mathcal L_{\mathrm{MW}} = L^2\sigma^2.
\end{equation}

\subsubsection{Naive master-weight removal}

\paragraph{Algorithm.}
We store only quantized weights $\hat x_t$. Each step:
\begin{equation}
    m_{t+1}=\beta m_t + (1-\beta)L\hat x_t,\qquad \tilde x_{t+1}=\hat x_t-\eta m_{t+1},\qquad
    \hat x_{t+1}=q(\tilde x_{t+1})=\tilde x_{t+1}+\xi_{t+1}.
\end{equation}

\paragraph{Linear form.}
With $c=(1-\beta)L$ and the same $a,b,d$ as above,
\begin{equation}
\label{eq:naive_lin}
\begin{aligned}
    \hat x_{t+1} &= a\hat x_t + b m_t + 1\cdot \xi_{t+1},\\
    m_{t+1} &= c\hat x_t + d m_t.
\end{aligned}
\end{equation}
This matches \eqref{eq:lin_form} with $(B_1,B_2)=(1,0)$ and state $x_t\equiv \hat x_t$.

\paragraph{Stationary second moments.}
Let $(u,v,w)$ denote the stationary solution for $u=\E[\hat x^2]$. Plugging $(B_1,B_2)=(1,0)$ into
Lemma~\ref{lem:moment_recursions} and setting stationarity yields
\begin{align}
    u &= a^2 u + 2ab v + b^2 w + \sigma^2, \label{eq:n_u}\\
    v &= ac\,u + (ad+bc)\,v + bd\,w, \label{eq:n_v}\\
    w &= c^2 u + 2cd\,v + d^2 w. \label{eq:n_w}
\end{align}
From \eqref{eq:n_w} and $d=\beta$,
\begin{equation}
\label{eq:n_w_solve}
    (1-\beta^2)w = c^2 u + 2c\beta v
    \quad\Longrightarrow\quad
    w=\frac{c^2u + 2c\beta v}{1-\beta^2}.
\end{equation}
Substitute \eqref{eq:n_w_solve} into \eqref{eq:n_v}; as above, $ad+bc=\beta-2\eta\beta c$ and $bd=-\eta\beta^2$.
This yields one linear equation in $(u,v)$, which solves to
\begin{equation}
\label{eq:n_v_solve}
    v = -\frac{\sigma^2(L\eta-\beta-1)}{\eta\big(2(1+\beta)-L\eta(1-\beta)\big)}.
\end{equation}
Plugging \eqref{eq:n_v_solve} into \eqref{eq:n_w_solve} gives $w$; substituting $(v,w)$ into \eqref{eq:n_u} and solving
for $u$ yields the closed form
\begin{equation}
\label{eq:n_u_final}
    u=\E[\hat x^2]
    =
    \sigma^2\,
    \frac{(1-\beta^2)+2\beta L\eta}{L\eta\big(2(1-\beta^2)-L\eta(1-\beta)^2\big)}.
\end{equation}

\paragraph{Divergence as $\eta\to 0$.}
From \eqref{eq:n_u_final}, as $\eta\to 0$ the denominator is $2L\eta(1-\beta^2)+o(\eta)$ while the numerator is
$(1-\beta^2)+o(1)$, hence
\begin{equation}
\label{eq:n_asymp}
    \E[\hat x^2] = \frac{\sigma^2}{2L\eta}+O(1),
    \qquad \eta\to 0.
\end{equation}
Therefore
\begin{equation}
\label{eq:n_limit}
    \mathcal L_{\mathrm{Naive}} = L^2\E[\hat x^2]
    \sim \frac{L\sigma^2}{2\eta}
    \xrightarrow[\eta\to 0]{}\infty.
\end{equation}

\subsubsection{ECO: momentum injection eliminates the $1/\eta$ blow-up}

\paragraph{Algorithm.}
ECO uses the same SGDM step as the naive method to compute $(\tilde x_{t+1},\tilde m_{t+1})$ from $(\hat x_t,\hat m_t)$,
then quantizes and injects the quantization error into momentum.
Concretely:
\begin{equation}
    \tilde m_{t+1}=\beta \hat m_t + (1-\beta)L\hat x_t,\qquad
    \tilde x_{t+1}=\hat x_t-\eta \tilde m_{t+1},\qquad
    \hat x_{t+1}=q(\tilde x_{t+1})=\tilde x_{t+1}+\xi_{t+1}.
\end{equation}
Define the (post-quantization) error $e_{t+1}\coloneqq \tilde x_{t+1}-\hat x_{t+1}=-\xi_{t+1}$. ECO then sets
\begin{equation}
\label{eq:eco_inject}
    \hat m_{t+1}=\tilde m_{t+1}+\alpha e_{t+1}
    = \tilde m_{t+1}-\alpha \xi_{t+1},
    \qquad
    \alpha=\frac{1}{\eta}\Big(1-\frac{1}{\beta}\Big)=\frac{\beta-1}{\eta\beta}.
\end{equation}
Since $\beta\in(0,1)$, $\alpha<0$. Define the positive injection gain
\begin{equation}
\label{eq:gamma_def}
    \gamma\coloneqq -\alpha=\frac{1-\beta}{\eta\beta}>0,
\end{equation}
so that $\hat m_{t+1}=\tilde m_{t+1}+\gamma\xi_{t+1}$.

\paragraph{Linear form.}
With $c=(1-\beta)L$ and the same $a,b,d$ as above, ECO becomes
\begin{equation}
\label{eq:eco_lin}
\begin{aligned}
    \hat x_{t+1} &= a\hat x_t + b \hat m_t + 1\cdot \xi_{t+1},\\
    \hat m_{t+1} &= c\hat x_t + d \hat m_t + \gamma\,\xi_{t+1},
\end{aligned}
\end{equation}
i.e., \eqref{eq:lin_form} with $(B_1,B_2)=(1,\gamma)$.

\paragraph{Stationary second moments.}
Applying Lemma~\ref{lem:moment_recursions} to \eqref{eq:eco_lin} and setting stationarity yields
\begin{align}
    u &= a^2 u + 2ab v + b^2 w + \sigma^2, \label{eq:e_u}\\
    v &= ac\,u + (ad+bc)\,v + bd\,w + \gamma\sigma^2, \label{eq:e_v}\\
    w &= c^2 u + 2cd\,v + d^2 w + \gamma^2\sigma^2. \label{eq:e_w}
\end{align}
We again eliminate $w$ using \eqref{eq:e_w} (same algebra as before) and then eliminate $v$ using \eqref{eq:e_v}.
The resulting expressions simplify dramatically because $\gamma$ is coupled to $(\eta,\beta)$ by \eqref{eq:gamma_def}.
Solving \eqref{eq:e_u}--\eqref{eq:e_w} yields the closed form
\begin{equation}
\label{eq:e_u_final}
    u=\E[\hat x^2]
    =
    \frac{2\sigma^2}{\,2(1-\beta^2)-L\eta(1-\beta)^2\,}.
\end{equation}

\paragraph{Finite noise floor as $\eta\to 0$.}
Taking $\eta\to 0$ in \eqref{eq:e_u_final} gives
\begin{equation}
\label{eq:e_limit_u}
    \lim_{\eta\to 0}\E[\hat x^2]=\frac{\sigma^2}{1-\beta^2},
\end{equation}
and therefore the stationary squared gradient satisfies
\begin{equation}
\label{eq:e_limit}
    \lim_{\eta\to 0}\mathcal L_{\mathrm{ECO}}
    = \lim_{\eta\to 0} L^2\E[\hat x^2]
    = \frac{L^2\sigma^2}{1-\beta^2}.
\end{equation}

\paragraph{Interpretation.}
Comparing \eqref{eq:n_limit} and \eqref{eq:e_limit}, naive master-weight removal yields a stationary error that blows up
like $1/\eta$, while ECO stabilizes the dynamics and yields a finite noise floor controlled by the geometric factor
$1/(1-\beta^2)$.
\end{document}